\documentclass{article}

\PassOptionsToPackage{numbers,compress}{natbib}
\usepackage[preprint]{neurips_2026}
\usepackage{float}

\usepackage[utf8]{inputenc} % allow utf-8 input
\usepackage[T1]{fontenc}    % use 8-bit T1 fonts
\usepackage{hyperref}       % hyperlinks
\usepackage{url}            % simple URL typesetting
\usepackage{booktabs}       % professional-quality tables
\usepackage{amsfonts}       % blackboard math symbols
\usepackage{nicefrac}       % compact symbols for 1/2, etc.
\usepackage{microtype}      % microtypography
\usepackage[table]{xcolor}  % colors
\usepackage{wrapfig}
\usepackage{amsthm}

\newtheorem{proposition}{Proposition}

\usepackage{subcaption}
\usepackage{graphicx}
\usepackage{amsmath}
\usepackage[ruled,vlined]{algorithm2e}
\usepackage{multirow}
\usepackage{enumitem}
\usepackage{lipsum}
\usepackage{array}
\usepackage{colortbl}

\definecolor{cardinal}{HTML}{8C1515}

\definecolor{tblHeader}{HTML}{F1F5F9}   % Very light gray-blue (header row)
\definecolor{tblGroupA}{HTML}{FEF8F5}   % Very light warm peach (No-RM)
\definecolor{tblGroupB}{HTML}{F3F9F5}   % Very light cool mint (Reward model)
\definecolor{tblGroupC}{HTML}{F3F7FC}   % Very light airy blue (Online RL)
\definecolor{tblRule}{HTML}{C8D0DA}     % soft rule color
\renewcommand{\arraystretch}{0.9}

\title{General Preference Reinforcement Learning}

\author{
  Muhammad Umer\thanks{Equal contribution.} \\
  Stanford University \\
  \texttt{mumer@stanford.edu} \\
  \And
  Muhammad Ahmed Mohsin\footnotemark[1] \\
  Stanford University \\
  \texttt{muahmed@stanford.edu}
  \And
  Ahsan Bilal \\
  The University of Oklahoma \\
  \texttt{ahsan.bilal-1@ou.edu}
  \And
  Arslan Chaudhry \\
  \texttt{} \\
  \texttt{}
  \And
  Andreas Haupt \\
  Stanford University \\
  \texttt{h4upt@stanford.edu}
  \And
  Sanmi Koyejo \\
  Stanford University \\
  \texttt{sanmi@stanford.edu}
  \And
  Emily Fox \\
  Stanford University \\
  \texttt{ebfox@stanford.edu}
  \And
  John M. Cioffi \\
  Stanford University \\
  \texttt{cioffi@stanford.edu}
}

\begin{document}
\maketitle
\begin{abstract}
Post-training has split large language model (LLM) alignment into two largely disconnected tracks. Online reinforcement learning (RL) with verifiable rewards drives emergent reasoning on math and code but depends on a programmatic verifier that cannot reach open-ended tasks, while preference optimization handles open-ended generation yet forgoes the continuous exploration that powers online RL. Closing this gap requires a verifier for open-ended quality, but a scalar reward model is the wrong shape for the job. Quality is multi-dimensional, and any scalar score is an incomplete proxy that lets online RL collapse onto whichever axis the score is most sensitive to. We turn instead to the General Preference Model (GPM), which embeds responses into $k$ skew-symmetric subspaces and represents preference as a structured, intransitivity-aware comparison. Building on this, we propose General Preference Reinforcement Learning (GPRL), which carries the $k$-way structure through to the policy update. GPRL computes per-dimension group-relative advantages, normalizes each on its own scale so no axis can dominate, and aggregates them with context-dependent eigenvalues. The same structure powers a closed-loop drift monitor that detects single-axis exploitation and corrects it on the fly by reweighting dimensions and tightening the trust region. Starting from $\texttt{Llama-3-8B-Instruct}$, GPRL reaches a length-controlled win rate of $56.51\%$ on AlpacaEval~2.0 while also outperforming SimPO and SPPO on Arena-Hard, MT-Bench, and WildBench by resisting reward hacking across extended training runs.
% Code for this work is available at \url{`TODO/released_upon_publication'}.
\end{abstract}

\section{Introduction}
Post-training now decides what a large language model (LLM) can actually do, since it turns a pretrained base model into a system that follows instructions, reasons through hard problems, and aligns with human values~\citep{kumar2025llm}. The dominant recipe, Reinforcement Learning from Human Feedback (RLHF)~\citep{ouyang2022training}, fits a scalar reward model (RM) to human preferences and optimizes a policy against it with Proximal Policy Optimization (PPO)~\citep{schulman2017proximal}. Although this pipeline has unlocked capabilities such as instruction following, multi-turn dialogue, and tool use, it remains complex, trains unstably, and gives way to reward hacking once pushed at scale~\citep{shao2024deepseekmath, gao2023scaling}.

\begin{wrapfigure}[21]{r}{0.48\columnwidth}
    \centering
    \includegraphics[width=0.48\columnwidth]{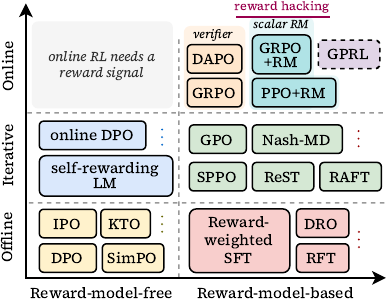}
    \caption{\textbf{Landscape of LLM post-training methods}, organized by supervision source and training regime. Online RL with a scalar RM reaches open-ended tasks but suffers reward hacking; GPRL fills the gap with a structured, multi-dimensional reward.}
    \label{fig:landscape}
\end{wrapfigure}

In response, the field has split into two largely disconnected tracks. The first avoids explicit reward modeling and optimizes the policy directly on preference data. Offline methods such as Direct Preference Optimization (DPO)~\citep{rafailov2023direct} dominate this track alongside game-theoretic variants such as Self-Play Preference Optimization (SPPO)~\citep{wu2024self}, Simple Preference Optimization (SimPO)~\citep{meng2024simpo}, and Nash-MD~\citep{munos2024nash}. These methods align well on open-ended tasks, yet they train on static or iteratively refreshed preference pairs and forgo the continuous online exploration that defines reinforcement learning (RL)~\citep{wu2024alphadpo,gou2024mixed}. The second track keeps an explicit reward signal and operates online, sampling fresh responses from the current policy at every step. Group Relative Policy Optimization (GRPO)~\citep{shao2024deepseekmath} has emerged as the standard algorithm here. It provides a critic-free variant of PPO that estimates advantages from a group of sampled responses and admits any scoring function as the reward. When the reward is a programmatic verifier such as a math checker or a unit test, the resulting reinforcement learning with verifiable rewards (RLVR) framework powers modern LLMs such as DeepSeek-R1~\citep{guo2025deepseek} and elicits emergent behaviors, including self-reflection and dynamic strategy adaptation, that supervised fine-tuning (SFT) alone does not produce.

A natural way to combine the two tracks is to run GRPO with a learned scalar RM in place of the verifier, and Figure~\ref{fig:landscape} shows where this combination sits. While the idea looks attractive on paper, since open-ended tasks admit no programmatic verifier and a learned RM seems to fill the gap, in practice it breaks down under extended online RL because of reward hacking. The deeper issue is that any scalar score is an \emph{incomplete proxy} for what humans want from a response, in the sense formalized by Zhuang et al.~\citep{zhuang2020consequences}. Quality has many latent dimensions including helpfulness, factual accuracy, safety, and style, and folding them into one number always discards information about how they trade off~\citep{zhong2024panacea}. Optimizing the proxy then amplifies whichever dimension the RM is most sensitive to and silently degrades the rest~\citep{gao2023scaling, guo2025deepseek}.

We trace the root cause to the lack of representational structure in scalar rewards, and find a principled alternative in the General Preference Model (GPM)~\citep{zhang2024beyond}, which embeds each response into $\mathbb{R}^{2k}$ and computes preferences with a block-diagonal skew-symmetric operator on $k$ independent two-dimensional subspaces, each capturing a different facet of quality.\footnote{The best $k$ is a property of the supervision corpus, not of GPM. With $k=3$ on \texttt{Skywork-Reward}~\citep{liu2024skywork}, the subspaces specialize on roughly distinguishable axes such as helpfulness against verbosity, factual accuracy against fluency, and safety against directness; we use $k=3$ throughout because it is where this corpus saturates.} This delivers a multi-dimensional preference signal, represents intransitive preferences ($A \succ B \succ C \succ A$) that scalar models cannot express, and preserves the linear query complexity of the Bradley-Terry (BT) model~\citep{bradley1952rank}.

Building on this, we propose General Preference Reinforcement Learning (GPRL), which treats GPM as a structured, multi-dimensional reward source for GRPO-style online RL on open-ended tasks. Rather than collapsing each response to a scalar advantage, GPRL computes \emph{per-dimension group-relative advantages} across GPM's $k$ subspaces, normalizes each on its own scale, and aggregates them with context-dependent eigenvalues. After rescaling, every subspace contributes on the same unit-variance scale, so no axis can inflate its share by simply growing in magnitude, and any policy that improves on one axis at the expense of others sees its aggregate advantage held in check by the rest. This vector-derived advantage replaces GRPO's scalar advantage in an otherwise identical objective. The same structure also lets us \emph{correct} reward hacking online through a \emph{drift monitor} that tracks the variance profile across preference dimensions during training, closing the loop with a controller that reweights the dimensions and tightens the KL trust region whenever drift crosses a threshold. Should the policy concentrate probability mass on, say, the verbosity axis, the controller downweights verbosity in the next aggregate advantage, and the policy's incentive shifts back toward neglected axes such as factuality and safety.

\textbf{Contributions.}
\textbf{(i)} We propose GPRL, which extends GRPO-style online RL to open-ended tasks through a multi-dimensional, intransitivity-aware preference embedding and recovers GRPO at $k{=}1$. Unlike multi-objective RLHF, which returns a Pareto family from independent scalar rewards~\citep{xiong2025projection,he2025pareto}, GPRL trains a single online policy whose axis structure comes directly from the GPM. \textbf{(ii)} We introduce per-dimension group-relative advantage estimation and identify a sufficient condition under which the aggregate advantage rejects single-axis exploitation that any scalar reward must prefer, formally separating GPRL from scalar GRPO. \textbf{(iii)} We design a closed-loop drift monitor that flags reward hacking from the shape of the advantage variance profile and corrects it online by reweighting dimensions and tightening the KL trust region, all without any external evaluation signal. \textbf{(iv)} Starting from \texttt{Llama-3-8B-Instruct}, GPRL reaches a $56.51\%$ length-controlled win rate on AlpacaEval~2.0, beating GRPO with a BT reward model by $14.59$ points and also outperforming DPO, SimPO, SPPO, GPO, and GRPO+BT on Arena-Hard, MT-Bench, and WildBench.

\section{Preliminaries}

\subsection{Group Relative Policy Optimization}

GRPO~\citep{shao2024deepseekmath} is a critic-free variant of PPO~\citep{schulman2017proximal} that estimates advantages from a group of sampled responses rather than a learned value function. Given a prompt $x$, GRPO samples $G$ responses $\{y_1, \ldots, y_G\}$ from $\pi_{\theta_{\text{old}}}(\cdot \mid x)$, scores each with a reward $R$, and normalizes within the group to form advantages $\hat{A}_i = \big(R(x, y_i) - \operatorname{mean}_{j} R(x, y_j)\big) / \big(\operatorname{std}_{j} R(x, y_j) + \epsilon\big)$. With importance ratio $r_i = \pi_\theta(y_i \mid x) / \pi_{\theta_{\text{old}}}(y_i \mid x)$ and its clipped counterpart $\bar{r}_i = \operatorname{clip}(r_i, 1{-}\epsilon, 1{+}\epsilon)$, GRPO minimizes
\begin{equation}
    \mathcal{L}_{\text{GRPO}}(\theta) = -\mathbb{E}_{x, \{y_i\}}  \left[ \frac{1}{G} \sum_{i=1}^{G} \min \big( r_i \hat{A}_i, \bar{r}_i \hat{A}_i \big) - \beta  \mathrm{KL} \big(\pi_\theta  \|  \pi_{\text{ref}}\big) \right] ,
    \label{eq:grpo}
\end{equation}
where $\beta$ controls KL regularization toward a reference policy $\pi_{\text{ref}}$. In RLVR~\citep{guo2025deepseek,zhang2025extending}, $R$ is a binary verifier $R(x, y) \in \{0, 1\}$ that checks correctness, which works well for math and code but does not extend to open-ended tasks; swapping in a learned scalar RM reintroduces reward hacking~\citep{gao2023scaling,fu2025reward}.

\subsection{General Preference Model and General Preference Optimization}
\label{sec:gpm}

The BT model~\citep{bradley1952rank} sets $\mathbb{P}(y_i \succ y_j \mid x) = \sigma(R(x, y_i) - R(x, y_j))$ with $\sigma$ the logistic function and $R(\cdot) \in \mathbb{R}$ a scalar reward. This formulation imposes a total ordering and cannot represent intransitive preferences such as $A \succ B \succ C \succ A$ that appear in human judgments and rule out any scalar utility~\citep{fishburn1982nontransitive,fishburn1991nontransitive}. The classical fix models preferences with a \emph{skew-symmetric bilinear} form on a convex set of options, generalizing von Neumann-Morgenstern utility and admitting intransitive cycles by construction~\citep{fishburn1981axiomatic, nakamura1998skew}. GPM~\citep{zhang2024beyond} brings this idea into modern alignment by instantiating the bilinear form via learned response embeddings.

\textbf{Preference embeddings.} GPM maps each response $y$ to an embedding $\mathbf{v}_{y \mid x} \in \mathbb{R}^{2k}$ with $\|\mathbf{v}_{y \mid x}\|_2 = 1$, scoring pairs through a fixed block-diagonal skew-symmetric operator $\mathbf{R}^{\succ} \in \mathbb{R}^{2k \times 2k}$ consisting of $k$ blocks $\mathbf{R}_l = \big[\begin{smallmatrix} 0 & -1 \\ 1 & \phantom{-}0 \end{smallmatrix}\big]$. Writing the embedding components as $(v_i^{(1)}, \ldots, v_i^{(2k)})$, the per-subspace score is
\begin{equation}
    s_l(y_i, y_j \mid x) = v_i^{(2l)} v_j^{(2l-1)} - v_i^{(2l-1)} v_j^{(2l)},
    \label{eq:subspace_score}
\end{equation}
and the overall score aggregates across subspaces with context-dependent eigenvalues\footnote{Real skew-symmetric matrices have purely imaginary eigenvalues in conjugate pairs $\pm i \lambda_l$; the $\lambda_l(x) \geq 0$ in Eq.~\eqref{eq:gpm_agg} are their moduli. They are produced for each prompt by a learned ``eigenvalue scale gate'' on the language model's prompt encoding~\citep{zhang2024beyond}.} $\lambda_l(x) \geq 0$,
\begin{equation}
    s(y_i \succ y_j \mid x) = \sum_{l=1}^{k} \lambda_l(x)  s_l(y_i, y_j \mid x), \quad \mathbb{P}(y_i \succ y_j \mid x) = \sigma\big(s(y_i \succ y_j \mid x)\big).
    \label{eq:gpm_agg}
\end{equation}
Since Eq.~\eqref{eq:subspace_score} is antisymmetric in $(i, j)$, the aggregate score is too, and the resulting preference probabilities define a constant-sum two-player game~\citep{wu2024self,munos2024nash}.

\textbf{General Preference Optimization.} Zhang et al.~\citep{zhang2024beyond} pair GPM with an iterative SPPO-style scheme called General Preference Optimization (GPO). Given an opponent policy $\mu$, the algorithm samples $K$ responses per prompt, estimates $\hat{s}(y_i \succ \mu \mid x) = \tfrac{1}{K} \sum_{k=1}^{K} s(y_i \succ y_k \mid x)$, and regresses $\log \pi_\theta / \pi_{\theta_t}$ onto these scores at each iteration $t$. The procedure remains offline and iterative, since one fixes $\pi_{\theta_t}$, collects a batch, trains to convergence, and only then refreshes the rollout policy, mirroring the regime that limits SPPO and DPO-style methods~\citep{song2024importance,wang2025indirect}. GPRL keeps GPM as the reward source but replaces GPO's iterative regression with an online policy-gradient update that samples freshly from $\pi_\theta$ at every step, computes per-dimension group-relative advantages instead of collapsing to the scalar $\hat{s}$, and adds a mechanism to control drift.

\section{General Preference Reinforcement Learning}
\label{sec:method}

GPRL replaces GRPO's scalar reward with GPM's multi-dimensional preference signal and carries the $k$-way structure through to the policy update. Averaging the $k$ subspace scores into a single number would discard the very structure that makes GPM worth using, so GPRL keeps each subspace separate during advantage estimation, normalizes each on its own scale, and only then aggregates.

\begin{figure}
    \centering
    \includegraphics[width=0.99\textwidth]{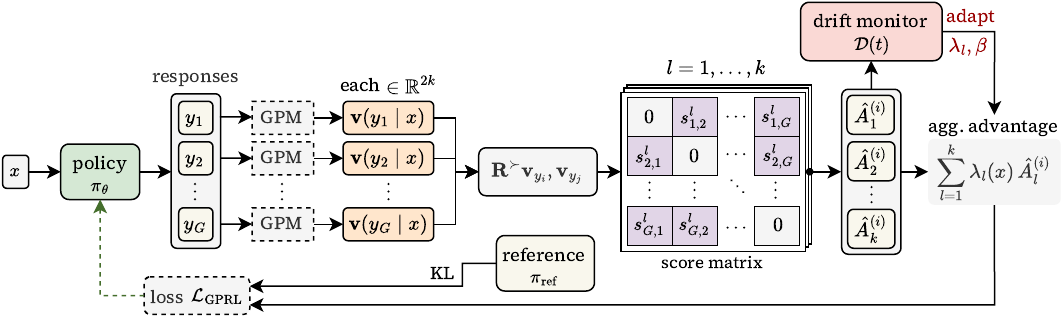}
    \caption{\textbf{Overview of GPRL.} The policy $\pi_\theta$ samples $G$ responses per prompt, GPM embeds them, and $\mathbf{R}^{\succ}$ produces $k$ pairwise score matrices that yield per-dimension advantages. The aggregate drives the GRPO-style clipped surrogate, while a drift monitor $\mathcal{D}(t)$ adapts the dimensional weights and $\beta$ to suppress reward hacking.}
    \label{fig:method}
\end{figure}

\subsection{Per-dimension advantage estimation and GPRL objective}
\label{sec:advantage}

Given a prompt $x$, GPRL samples $G$ responses $\{y_1, \ldots, y_G\}$ from $\pi_{\theta_{\text{old}}}$ and passes them through a frozen GPM, which returns embeddings $\mathbf{v}_{y_i \mid x} \in \mathbb{R}^{2k}$ and eigenvalues $\lambda_l(x)$. We then build the advantage in three steps. First, we form the $k$ pairwise score matrices using $s_l(y_i, y_j \mid x)$ from Eq.~\eqref{eq:subspace_score}, as shown in Figure~\ref{fig:method}. Averaging each row against the rest of the group yields a per-dimension population score $\hat{s}_l^{(i)}(x) = \tfrac{1}{G-1}\sum_{j \neq i} s_l(y_i, y_j \mid x)$ that measures how much $y_i$ beats the rest of the group along dimension $l$. Second, since different subspaces operate at different scales, we normalize within each dimension rather than across all of them,
\begin{equation}
    \hat{A}_l^{(i)}(x) = \frac{\hat{s}_l^{(i)}(x) - \mu_l(x)}{\sigma_l(x) + \epsilon},
    \quad
    \hat{A}_l^{(i)}(x) = 0 \text{ if } \sigma_l(x) = 0,
    \label{eq:gprl_per_dim}
\end{equation}
where $\mu_l(x)$ and $\sigma_l(x)$ are the mean and standard deviation of $\{\hat{s}_l^{(i)}(x)\}_{i=1}^G$. A single global normalization would let whichever subspace carries the largest magnitudes drown out the rest, whereas per-dimension rescaling places every signal on a common unit-variance footing. Finally, we combine the rescaled advantages with GPM's eigenvalues to form the aggregate advantage,
\begin{equation}
    \hat{A}^{(i)}(x) = \sum_{l=1}^{k} \lambda_l(x)  \hat{A}_l^{(i)}(x),
    \label{eq:gprl_advantage}
\end{equation}
and plug this into a GRPO-style clipped surrogate to obtain the GPRL objective,
\begin{equation}
    \mathcal{L}_{\text{GPRL}}(\theta) = -\mathbb{E}_{x, \{y_i\}} \left[ \frac{1}{G}\sum_{i=1}^{G} \min \big( r_i \hat{A}^{(i)}, \bar{r}_i \hat{A}^{(i)} \big) - \beta  \mathrm{KL} \big(\pi_\theta  \|  \pi_{\text{ref}}\big) \right] ,
    \label{eq:gprl_loss}
\end{equation}
where $r_i$ and $\bar{r}_i$ match those of Eq.~\eqref{eq:grpo}. Since Eq.~\eqref{eq:gprl_loss} is structured identically to the GRPO objective and only the computation of $\hat{A}^{(i)}$ differs, GPRL drops into existing large-scale RL infrastructure such as vLLM rollouts~\citep{kwon2023efficient} and distributed training seamlessly, with the only added cost being a frozen GPM forward pass per group that is comparable to any learned RM call.

\begin{proposition}[Zero-mean advantages]
\label{prop:zero_mean}
For every prompt $x$ and every dimension $l$, $\sum_{i=1}^{G} \hat{A}_l^{(i)}(x) = 0$, hence $\sum_{i=1}^{G} \hat{A}^{(i)}(x) = 0$ for any weights $\lambda_l(x) \in \mathbb{R}$.
\end{proposition}
Proposition~\ref{prop:zero_mean} implies that aggregating $k$ subspace advantages does not break the variance-reduction property that makes group-relative methods well-behaved~\citep{shao2024deepseekmath}. The aggregate retains a zero baseline within every group, the policy gradient stays unbiased, and we are free to mix dimensions through any weights $\lambda_l(x)$ without rederiving anything.\footnote{Fixing $k = 1$ and $\lambda_1(x) \equiv 1$ and defining $R_{\mathrm{GPM}}(x, y_i) := \hat{s}_1^{(i)}(x)$ collapses Eq.~\eqref{eq:gprl_advantage} to the GRPO advantage under $R_{\mathrm{GPM}}$ and Eq.~\eqref{eq:gprl_loss} to Eq.~\eqref{eq:grpo}, so GPRL strictly generalizes GRPO under a preference-model-based scalar reward, and any improvement we observe with $k > 1$ is attributable to the multi-dimensional structure rather than a different optimizer.} Proofs for this and what follows appear in Appendix~\ref{app:proofs}.

\begin{figure}[t]
    \centering
    \begin{subfigure}[t]{0.3\textwidth}
        \centering
        \includegraphics[width=\linewidth]{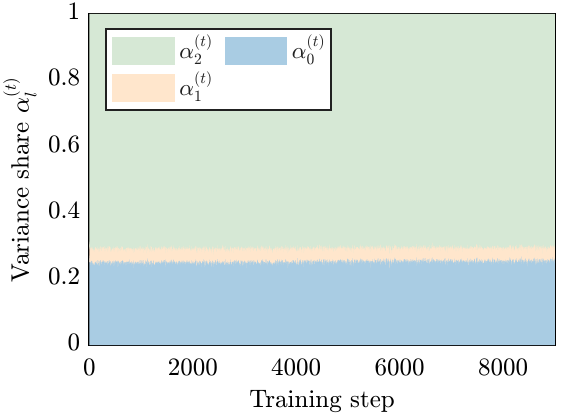}
        \caption{Healthy run}
        \label{fig:drift:healthy}
    \end{subfigure}
    \hfill
    \begin{subfigure}[t]{0.3\textwidth}
        \centering
        \includegraphics[width=\linewidth]{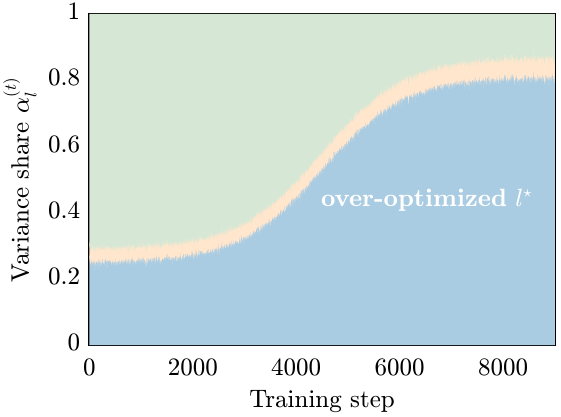}
        \caption{Hacked run}
        \label{fig:drift:hacked}
    \end{subfigure}
    \hfill
    \begin{subfigure}[t]{0.33\textwidth}
        \centering
        \includegraphics[width=\linewidth]{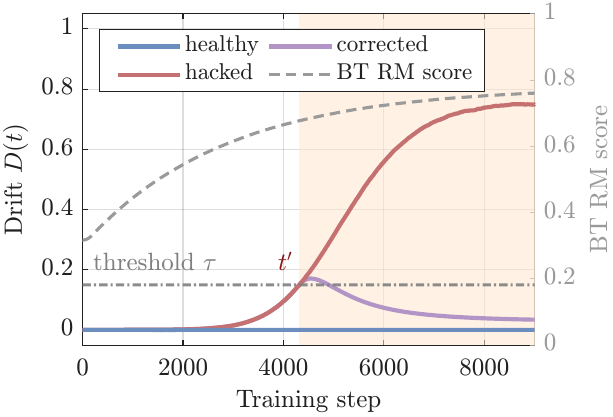}
        \caption{Drift trajectory}
        \label{fig:drift:trajectory}
    \end{subfigure}
    \caption{\textbf{Dimensional drift distinguishes healthy training from reward hacking.} (\subref{fig:drift:healthy}) The variance profile $\alpha^{(t)}$ holds its initial shape on a healthy GPRL run. (\subref{fig:drift:hacked}) Under hacking, it collapses onto a single dimension $l^\star$. (\subref{fig:drift:trajectory}) $\mathcal{D}(t)$ stays near zero on the healthy run and crosses $\tau$ at $t'$ on the hacked one, allowing the \emph{corrected} trajectory to engage the controller at $t'$ and pull back as the profile rebalances, while a BT reward model climbs monotonically and offers no comparable signal.}
    \label{fig:drift}
\end{figure}

\subsection{Drift control}
\label{sec:drift}

The aggregate advantage of Eq.~\eqref{eq:gprl_advantage} acts as a lever the policy gradient pulls on, so understanding why a multi-dimensional reward resists reward hacking requires asking what it takes for the lever to flip the wrong way. Suppose a candidate policy $\pi^{\dagger}$ improves a single axis $l^{\ast}$ relative to a reference policy $\pi^{\ast}$ while regressing on the rest. The gradient direction depends on which policy carries the larger aggregate advantage, and Proposition~\ref{prop:hacking} pins down when GPRL prefers the balanced policy.

\begin{proposition}[Multi-dimensional structure rejects single-axis hacking]
\label{prop:hacking}
Fix a prompt $x$ and a group of responses, and let $\hat{A}_l^{\pi}(x)$ denote the dimension-$l$ advantage of a response drawn from $\pi$, as in Eq.~\eqref{eq:gprl_per_dim}. Let $\pi^{\ast}$ and $\pi^{\dagger}$ satisfy, for some axis $l^{\ast}$, $\hat{A}_{l^{\ast}}^{\pi^{\dagger}}(x) > \hat{A}_{l^{\ast}}^{\pi^{\ast}}(x)$ and $\hat{A}_{l}^{\pi^{\dagger}}(x) < \hat{A}_{l}^{\pi^{\ast}}(x)$ for all $l \neq l^{\ast}$. If
\begin{equation}
    \sum_{l \neq l^{\ast}} \lambda_l(x)  \big(\hat{A}_l^{\pi^{\ast}}(x) - \hat{A}_l^{\pi^{\dagger}}(x)\big) > \lambda_{l^{\ast}}(x)  \big(\hat{A}_{l^{\ast}}^{\pi^{\dagger}}(x) - \hat{A}_{l^{\ast}}^{\pi^{\ast}}(x)\big),
    \label{eq:hack_condition}
\end{equation}
then $\hat{A}^{\pi^{\dagger}}(x) < \hat{A}^{\pi^{\ast}}(x)$.
\end{proposition}

The verdict $\hat{A}^{\pi^{\dagger}}(x) < \hat{A}^{\pi^{\ast}}(x)$ is the consequential half. GPRL gives the smaller aggregate advantage to the single-axis exploiter, so the clipped surrogate in Eq.~\eqref{eq:gprl_loss} pushes the policy gradient away from $\pi^{\dagger}$ and toward $\pi^{\ast}$. Under any scalar reward $R$ with $R(x, \pi^{\dagger}) > R(x, \pi^{\ast})$, by contrast, the GRPO advantage is monotone in $R$ within a fixed group, so $\pi^{\dagger}$ always wins and the gradient always rewards the exploit. The per-dimension normalization in Eq.~\eqref{eq:gprl_per_dim} is what makes Eq.~\eqref{eq:hack_condition} likely to hold, since rescaling every $\hat{A}_l^{(i)}$ to unit variance bounds how much any single dimension can grow its contribution to the aggregate. We use Eq.~\eqref{eq:hack_condition} as the design foundation of the controller below, ensuring that whenever the inequality fails, the controller intervenes on the very weights $\lambda_l(x)$ that govern it.

While Eq.~\eqref{eq:hack_condition} can be verified response by response, acting online requires a scalar that summarizes the inequality across a training step. We use the variance profile of the per-dimension advantages, since when a policy starts reward hacking $l^{\ast}$, the between-response variance concentrates on $l^{\ast}$ and collapses on the rest, which is exactly the imbalance the controller has to push back on. Define
\begin{equation}
    \alpha_l^{(t)} = \frac{\mathrm{Var}_{y_i \sim \pi_t}\big[\hat{s}_l^{(i)}(x)\big]}{\sum_{l'=1}^{k} \mathrm{Var}_{y_i \sim \pi_t}\big[\hat{s}_{l'}^{(i)}(x)\big]}, \quad l = 1, \ldots, k, \quad \mathcal{D}(t) = \mathrm{KL}\big(\alpha^{(t)} \,\|\, \alpha^{(0)}\big),
\end{equation}
with $\alpha^{(t)} = \mathbf{1}/k$ in the degenerate all-zero case. In healthy training $\alpha^{(t)}$ holds its initial shape, but under hacking it spikes on the exploited axis, causing $\mathcal{D}(t)$ to rise accordingly. We feed $\mathcal{D}(t)$ back into a closed-loop controller that retunes the aggregation weights and the KL coefficient, and Figure~\ref{fig:drift} sketches the qualitative drift trajectories it acts on.

\textbf{Controller.} The controller acts through a per-dimension multiplier $m_l(t) \geq 0$, initialized to $1$ and applied on top of the eigenvalues, so the effective weight becomes $\tilde{\lambda}_l(x, t) = m_l(t)\,\lambda_l(x)$. While $\mathcal{D}(t) > \tau$ for threshold $\tau$, the controller tightens according to
\begin{equation}
    m_l \leftarrow \mathcal{N}\left[m_l \cdot \left(\frac{\alpha_l^{(0)}}{\alpha_l^{(t)} + \varepsilon}\right)^{\gamma}\right], \quad \beta \leftarrow \min\big(\kappa \cdot \beta, \beta_{\max}\big),
    \label{eq:drift_tighten}
\end{equation}
where $\mathcal{N}[\cdot]$ renormalizes the multipliers to mean $1$, $\gamma \in (0, 1]$ sets the redistribution strength, and $\kappa > 1$ tightens the trust region. The update targets dimensions that need adjusting. Over-grown axes receive $m_l < 1$ and stagnated axes receive $m_l > 1$, so the aggregate steers back toward balance while the increase in $\beta$ slows further excursions in the meantime. Once $\mathcal{D}(t) \leq \tau$, the controller relaxes via $m_l \leftarrow \delta  m_l + (1 - \delta)$ and $\beta \leftarrow \max(\beta_0, \beta \cdot \delta)$ with recovery rate $\delta \in (0, 1)$, allowing both quantities to decay toward their baselines as the profile rebalances.
% \footnote{We do not claim $\mathcal{D}(t)$ is monotonically driven below $\tau$, but rather that the loop provides a pathway for the multi-dimensional signal of Proposition~\ref{prop:hacking} to push back on the policy update, which a scalar RM cannot supply since improvement from single-axis exploitation and improvement from balanced progress look identical through one number.}

\section{Experiments}
\label{sec:experiments}

\begin{table}[t]
\centering
\small
\caption{\textbf{AlpacaEval 2.0.} LC.~WR, WR, and average response length against the \texttt{gpt-4-turbo} reference, using \texttt{Llama-3-8B-Instruct} as the base policy and \texttt{gpt-4-turbo} as judge. SPPO and GPO follow the published three-iteration protocol, while GRPO and GPRL are reported at the matched third epoch so that all reward-model-based methods are compared at similar compute.}
\label{tab:alpacaeval}
\setlength{\tabcolsep}{6pt}
\begin{tabular}{llccccccc}
\toprule
\rowcolor{tblHeader}
 & & \multicolumn{2}{c}{\textbf{RM}} & & & & \\
\rowcolor{tblHeader}
\multirow{-2}{*}{\textbf{Group}} & \multirow{-2}{*}{\textbf{Method}} & \textbf{Size} & \textbf{Type} & \textbf{Iter./Ep.} & \multirow{-2}{*}{\textbf{LC.~WR}} & \multirow{-2}{*}{\textbf{WR}} & \multirow{-2}{*}{\textbf{Avg. Len}} \\
\midrule
\rowcolor{tblGroupA}
\cellcolor{white} & DPO~\citep{rafailov2023direct} & -- & -- & -- & 40.30 & 37.90 & 1837 \\
\rowcolor{tblGroupA}
\cellcolor{white}\multirow{-2}{*}{\textit{No reward model}} & SimPO~\citep{meng2024simpo} & -- & -- & -- & 44.70 & 40.50 & 1825 \\
\midrule
\rowcolor{tblGroupB}
\cellcolor{white}& & 2B & BT  & 3 & 40.01 & 42.12 & 2136 \\
\rowcolor{tblGroupB}
\cellcolor{white}& & 2B & GPM & 3 & 36.06 & 45.61 & 2498 \\
\rowcolor{tblGroupB}
\cellcolor{white}& & 8B & BT  & 3 & 42.55 & 40.92 & 1948 \\
\rowcolor{tblGroupB}
\cellcolor{white}& \multirow{-4}{*}{SPPO~\citep{wu2024self}} & 8B & GPM & 3 & 39.45 & 41.64 & 2385 \\
\cmidrule(lr){2-8}
\rowcolor{tblGroupB}
\cellcolor{white}& & 2B & BT  & 3 & 42.21 & 44.20 & 2151 \\
\rowcolor{tblGroupB}
\cellcolor{white}& & 2B & GPM & 3 & 37.74 & 48.25 & 2582 \\
\rowcolor{tblGroupB}
\cellcolor{white}& & 8B & BT  & 3 & 40.37 & 38.56 & 1969 \\
\rowcolor{tblGroupB}
\cellcolor{white}\multirow{-8}{*}{\textit{Reward model}} & \multirow{-4}{*}{GPO~\citep{zhang2024beyond}} & 8B & GPM & 3 & 38.98 & 41.54 & 3249 \\
\midrule
\rowcolor{tblGroupC}
\cellcolor{white}& & 2B & BT  & 3 & 39.87 & 38.21 & 1925 \\
\rowcolor{tblGroupC}
\cellcolor{white}& \multirow{-2}{*}{GRPO~\citep{shao2024deepseekmath}} & 8B & BT  & 3 & 41.92 & 40.51 & 1893 \\
\cmidrule(lr){2-8}
\rowcolor{tblGroupC}
\cellcolor{white}& & 2B & GPM & 3 & 51.08 & 45.21 & 1699 \\
\rowcolor{tblGroupC}
\cellcolor{white}\multirow{-4}{*}{\shortstack[l]{\textit{Online RL}\\\textit{with reward model}}} & \multirow{-2}{*}{\textbf{GPRL (ours)}} & 8B & GPM & 3 & \textbf{56.51} & \textbf{48.33} & \textbf{1600} \\
\bottomrule
\end{tabular}
\end{table}

\textbf{Setup.} Every method starts from the same base policy, \texttt{Llama-3-8B-Instruct}, so any difference in scores reflects the optimization signal rather than pretraining. To separate the effect of supervision structure from supervision content, we train two reward models on \texttt{Skywork-Reward}~\citep{liu2024skywork}, namely a scalar BT model and a GPM with embedding dimension $2k = 6$, each at two scales (\texttt{Gemma-2B-it} for 2B and \texttt{Llama-3.1-8B-Instruct} for 8B). The policy then trains under GPRL with online rollouts on prompts drawn from \texttt{UltraFeedback}~\citep{cui2023ultrafeedback}, using $G = 8$ generations per prompt, completion length $512$, and temperature $1.0$. We train for three epochs with AdamW at peak learning rate $1{\times}10^{-6}$, cosine schedule with $0.03$ warmup, weight decay $0.1$, gradient clipping at $1.0$, KL coefficient $\beta = 0.01$, and one inner iteration per rollout group. This setting matches the three-iteration style of SPPO and GPO so that all reward-model-based methods sit on a similar compute footing, and it also reaches the regime where the drift controller actively engages, as Appendix~\ref{app:ablation_scaling} characterizes further.

\textbf{Benchmarks.} We evaluate on three open-ended instruction-following benchmarks. AlpacaEval 2.0~\citep{dubois2024length} reports length-controlled win rate (LC.~WR) and WR against \texttt{gpt-4-turbo}\footnote{Our reported numbers come from an automated \texttt{gpt-4-turbo} judge rather than human raters, and the well-known length and stylistic biases of such judges apply to every method in our tables~\citep{dubois2024length,gu2024survey}.} on $805$ general prompts, with the LC variant explicitly debiasing against verbosity. Arena-Hard v2~\citep{li2024crowdsourced} draws $500$ adversarial queries from real Chatbot Arena traffic and currently shows the highest correlation with human Arena rankings among public benchmarks. WildBench (WB)~\citep{lin2024wildbench} scores $1024$ in-the-wild user instructions on a calibrated $0$ to $100$ WB-Score and a task-balanced WB-Reward in $[-100, 100]$; we additionally report MT-Bench~\citep{zheng2023judging} on $80$ multi-turn prompts.

\textbf{Baselines.} Five baselines span both post-training tracks. DPO~\citep{rafailov2023direct} reparameterizes RLHF as a closed-form classification loss on offline pairs, while SimPO~\citep{meng2024simpo} drops the reference model and uses a length-normalized log-likelihood margin. SPPO~\citep{wu2024self} iteratively regresses against pairwise win rates from a preference oracle, and GPO~\citep{zhang2024beyond} runs the same iterative recipe on GPM scores. GRPO~\citep{shao2024deepseekmath} paired with a scalar BT reward model is the same-paradigm online RL baseline that shares everything with GPRL except the multi-dimensional reward. Following published practice, SPPO and GPO are run for three iterations and we report iteration~3 results, since prior work establishes that this setting yields their strongest scores. GRPO and GPRL likewise report the third epoch against the matched 8B BT and 8B GPM reward models.

\begin{table}[t]
\centering
\small
\caption{\textbf{Arena-Hard v2 and MT-Bench.} Win rate (WR) against the Arena-Hard \texttt{gpt-4-0314} reference (AH2) and MT-Bench score, same setup as Table~\ref{tab:alpacaeval}.}
\label{tab:arenahard}
\setlength{\tabcolsep}{6pt}
\begin{tabular}{llccccc}
\toprule
\rowcolor{tblHeader}
  & & \multicolumn{2}{c}{\textbf{RM}} & & & \\
\rowcolor{tblHeader}
\multirow{-2}{*}{\textbf{Group}} & \multirow{-2}{*}{\textbf{Method}} & \textbf{Size} & \textbf{Type} & \textbf{Iter./Ep.} & \multirow{-2}{*}{\textbf{AH2}} & \multirow{-2}{*}{\textbf{MT-Bench}} \\
\midrule
\rowcolor{tblGroupA}
\cellcolor{white}& DPO~\citep{rafailov2023direct} & -- & -- & -- & 0.9 & 7.90 \\
\rowcolor{tblGroupA}
\cellcolor{white}\multirow{-2}{*}{\textit{No reward model}} & SimPO~\citep{meng2024simpo} & -- & -- & -- & 1.0 & 8.03 \\
\midrule
\rowcolor{tblGroupB}
\cellcolor{white}& & 8B & BT  & 3 & 0.9 & 8.02 \\
\rowcolor{tblGroupB}
\cellcolor{white}& \multirow{-2}{*}{SPPO~\citep{wu2024self}} & 8B & GPM & 3 & 0.9 & 8.12 \\
\cmidrule(lr){2-7}
\rowcolor{tblGroupB}
\cellcolor{white}& & 8B & BT  & 3 & 0.8 & 7.87 \\
\rowcolor{tblGroupB}
\cellcolor{white}\multirow{-4}{*}{\textit{Reward model}} & \multirow{-2}{*}{GPO~\citep{zhang2024beyond}} & 8B & GPM & 3 & 0.9 & 8.03 \\
\midrule
\rowcolor{tblGroupC}
\cellcolor{white}& GRPO~\citep{shao2024deepseekmath} & 8B & BT  & 3 & 0.8 & 7.69 \\
\cmidrule(lr){2-7}
\rowcolor{tblGroupC}
\cellcolor{white}\multirow{-2.3}{*}{\shortstack[l]{\textit{Online RL}\\\textit{with reward model}}} & \textbf{GPRL (ours)} & 8B & GPM & 3 & \textbf{1.3} & \textbf{8.33} \\
\bottomrule
\end{tabular}
\end{table}

\begin{table}[t]
\centering
\small
\caption{\textbf{WildBench.} Calibrated WB-Score and task-balanced WB-Reward against the three official baseline references (\texttt{gpt-4-turbo}, \texttt{claude-3-haiku}, \texttt{llama-2-70b-chat}), same setup as Table~\ref{tab:alpacaeval}.}
\label{tab:wildbench}
\setlength{\tabcolsep}{6pt}
\begin{tabular}{llccccc}
\toprule
\rowcolor{tblHeader}
 & & \multicolumn{2}{c}{\textbf{RM}} & & & \\
\rowcolor{tblHeader}
\multirow{-2}{*}{\textbf{Group}} & \multirow{-2}{*}{\textbf{Method}} & \textbf{Size} & \textbf{Type} & \textbf{Iter./Ep.} & \multirow{-2}{*}{\textbf{WB-Score}} & \multirow{-2}{*}{\textbf{WB-Reward}} \\
\midrule
\rowcolor{tblGroupA}
\cellcolor{white}& DPO~\citep{rafailov2023direct} & -- & -- & -- & 35.90 & 8.16 \\
\rowcolor{tblGroupA}
\cellcolor{white}\multirow{-2}{*}{\textit{No reward model}} & SimPO~\citep{meng2024simpo} & -- & -- & -- & 37.05 & 9.57 \\
\midrule
\rowcolor{tblGroupB}
\cellcolor{white}& & 8B & BT  & 3 & 36.14 & 8.33 \\
\rowcolor{tblGroupB}
\cellcolor{white}& \multirow{-2}{*}{SPPO~\citep{wu2024self}} & 8B & GPM & 3 & 35.21 & 7.18 \\
\cmidrule(lr){2-7}
\rowcolor{tblGroupB}
\cellcolor{white}& & 8B & BT  & 3 & 35.87 & 8.25 \\
\rowcolor{tblGroupB}
\cellcolor{white}\multirow{-4}{*}{\textit{Reward model}} & \multirow{-2}{*}{GPO~\citep{zhang2024beyond}} & 8B & GPM & 3 & 35.98 & 7.42 \\
\midrule
\rowcolor{tblGroupC}
\cellcolor{white}& GRPO~\citep{shao2024deepseekmath} & 8B & BT  & 3 & 35.41 & 8.62 \\
\cmidrule(lr){2-7}
\rowcolor{tblGroupC}
\cellcolor{white}\multirow{-2.3}{*}{\shortstack[l]{\textit{Online RL}\\\textit{with reward model}}} & \textbf{GPRL (ours)} & 8B & GPM & 3 & \textbf{37.98} & \textbf{11.15} \\
\bottomrule
\end{tabular}
\end{table}

\textbf{Main results.} Tables~\ref{tab:alpacaeval}, \ref{tab:arenahard}, and~\ref{tab:wildbench} report the scores, and Appendix~\ref{app:ablation_scaling} traces the full per-epoch trajectory. Three patterns repeat across all four benchmarks: \emph{(a) Static and iteratively-refreshed methods plateau early.} DPO and SimPO sit at $40$ to $45$ LC.~WR on AlpacaEval~2.0, while SPPO and GPO at iteration~3 stay in the same band. They cluster around $0.8$ to $0.9$ on Arena-Hard v2,\footnote{AH2 absolute scores are uniformly low because the base model is not strong on adversarial reasoning regardless of post-training. The ordering is what we test, and Appendix~\ref{app:ablation_basepolicy} reports stronger base policies.} $7.9$ to $8.1$ on MT-Bench, and $35$ to $37$ on WildBench. \emph{(b) Online RL with a scalar BT reward improves and then unwinds.} GRPO+BT peaks earlier in training than the third epoch we report and then regresses to $41.92$ LC.~WR on AlpacaEval~2.0, $7.69$ on MT-Bench, and $35.41$ on WildBench, indicative of a scalar proxy collapsing onto whichever axis it tracks most sensitively. \emph{(c) Swapping the scalar reward for GPM under the same online paradigm produces the largest jump and continues to improve through epoch~3.} GPRL using the 8B GPM reaches $56.51$ LC.~WR on AlpacaEval~2.0, beating GRPO+BT by $14.59$ points and the strongest iterative baseline (SPPO with 8B BT at $42.55$ LC.~WR) by $13.96$ points on the same metric, while the same ordering holds on the other three benchmarks.

The length numbers in Table~\ref{tab:alpacaeval} also tell a consistent story. Since the \texttt{gpt-4-turbo} judge favors longer responses~\citep{dubois2024length,zheng2023judging,shi2025judging}, iterative GPM-driven methods exploit this drift and inflate to $2400$ to $3300$ tokens by iteration~3 even as their LC win rates stall. GPRL, by contrast, holds length nearly flat at $1600$ tokens, the shortest of any reward-model-based method in the table. This is what per-dimension normalization (Eq.~\eqref{eq:gprl_per_dim}) is designed to produce, since after rescaling, no single subspace can dominate the aggregate by simply growing in magnitude, so any axis that correlates with length cannot pull the policy gradient on its own as long as the others retain comparable variance. The drift controller reinforces this by reweighting whichever subspace's variance share grows past the initial profile, length-correlated or otherwise, and collectively delivers the strongest LC.~WR in the table at the shortest average length, rather than the usual trade-off where one of the two has to give for the other to improve.

\subsection{Analysis and ablations}
\label{sec:ablations}

\begin{figure}[t]
\centering
\begin{subfigure}[b]{0.31\textwidth}
    \centering
    \includegraphics[width=\textwidth]{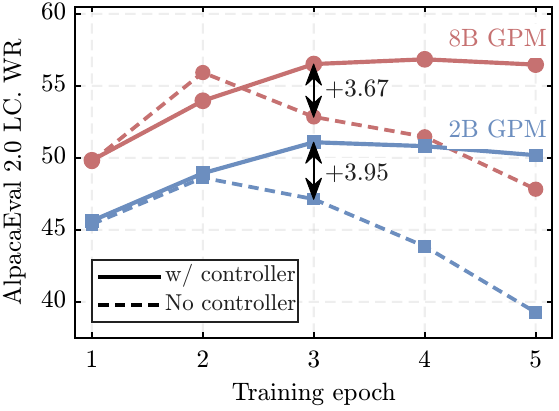}
    \caption{Extended-training scaling}
    \label{fig:scaling}
\end{subfigure}
\hspace{1em}
\begin{subfigure}[b]{0.3\textwidth}
    \centering
    \includegraphics[width=\textwidth]{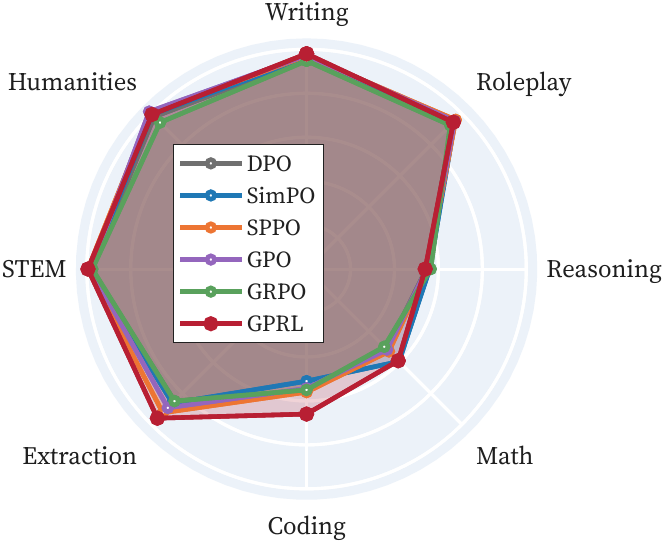}
    \caption{MT-Bench, per-category}
    \label{fig:mt_radar}
\end{subfigure}
\hspace{1em}
\begin{subfigure}[b]{0.3\textwidth}
    \centering
    \includegraphics[width=\textwidth]{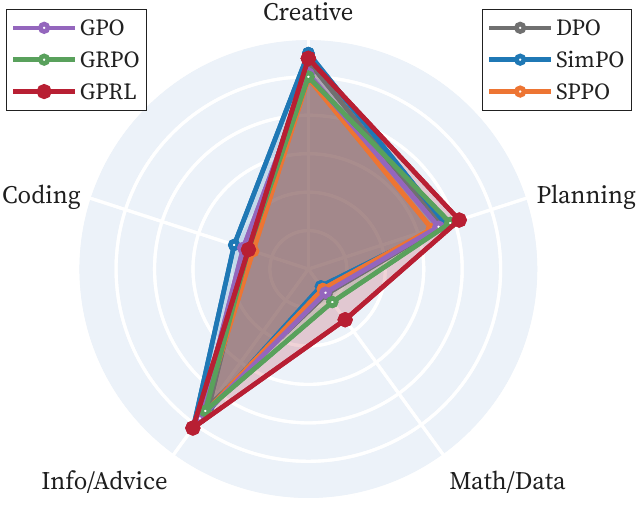}
    \caption{WildBench, per-category}
    \label{fig:wild_radar}
\end{subfigure}
\caption{\textbf{Scaling and per-category breakdown.} (\subref{fig:scaling}) AlpacaEval~2.0 LC.~WR across five training epochs at both reward-model scales, with the controller enabled holding near its peak through epoch~5 and the controller disabled degrading once drift develops. (\subref{fig:mt_radar},~\subref{fig:wild_radar}) Per-category scores on MT-Bench and WildBench, where GPRL leads on the categories that match the supervision and on structural categories while remaining within noise of the strongest baseline elsewhere.}
\label{fig:radar_charts}
\end{figure}

\textbf{Per-category gains and behavior under extended training.} Figure~\ref{fig:radar_charts} pairs the per-category breakdowns of MT-Bench and WildBench with the controller's behavior across epochs. Less expected is that from supervision on \texttt{Skywork-Reward} and rollouts on \texttt{UltraFeedback} GPRL's largest single-category gains land on the structural categories that lie outside both datasets, namely MT-Bench coding ($+1.00$ over the strongest iterative baseline) and WildBench math/data ($+2.84$ over GRPO+BT), with a comparable lead on extraction and a tie with SPPO on STEM. We attribute the pattern to the same mechanism that drives our main results, since online RL with group-relative advantages elicits emergent reasoning whenever the reward signal is rich enough to discriminate good from bad chains of thought~\citep{wang2025emergent}, and GPM's $k$ subspaces appear to provide enough discriminatory signal on the structural axes to keep the policy from collapsing onto stylistic mimicry on these prompts. This is consistent with reports from RLVR-style training~\citep{guo2025deepseek,shao2024deepseekmath} that online RL with a structured reward lifts the categories where the supervision discriminates well rather than uniformly improving every axis. Subfigure~\ref{fig:scaling} traces the same process across five epochs at both reward-model scales, where the controller-on runs hold within roughly half a point of their peak through epoch~5 while the controller-off runs degrade once drift develops.

\textbf{Ablations.} Appendix~\ref{app:ablations} provides the complete evaluation tables and analyses for each configuration; we summarize the primary findings here. (i) \textit{Number of subspaces $k$.} LC.~WR climbs from $44.21$ at $k = 1$, where GPM collapses to a single skew-symmetric block and GPRL recovers GRPO under the resulting BT-style scalar reward, to $56.51$ at $k = 3$. It then plateaus and slightly regresses by $k = 6$, with the jump from $k=1$ to $k=3$ being the largest single contributor to GPRL's gains, though the location of the plateau likely reflects \texttt{Skywork-Reward}'s facet structure rather than a universal optimum. (ii) \textit{Per-dimension vs.\ global normalization.} Replacing Eq.~\eqref{eq:gprl_per_dim} with a single global normalization that pools $(\mu, \sigma)$ across subspaces drops LC.~WR by roughly $4$ points and re-introduces length drift to $2104$ tokens, recovering the verbose pattern of the iterative GPM methods. (iii) \textit{Drift controller.} The controller has little effect at the first epoch, where the variance profile sits close to its initial shape and on/off scores stay within $0.59$ LC.~WR points, but it becomes load-bearing under extended training as the profile starts to lean on a single axis. By the third epoch it is worth $3.67$ LC.~WR points on the 8B GPM and $3.95$ on the 2B. Appendix~\ref{app:ablation_scaling} shows the gap widening to $8.65$ and $10.91$ points by epoch~5 as controller-off runs follow the same downward trajectory we observe for GRPO+BT. (iv) \textit{Drift threshold $\tau$.} Setting $\tau$ very low underperforms our default and even the disabled controller, because the controller starts redistributing weight before the policy has had a chance to legitimately concentrate variance on whichever subspace the prompt distribution rewards, and aggressive early correction drags the policy back toward the initialization profile and loses real signal alongside the spurious one. (v) \textit{Group size $G$.} Larger $G$ improves the variance estimates the drift monitor depends on, with gains saturating at $G = 8$, our default.

\textbf{Limitations.} GPRL inherits the assumptions of the GPM it consumes, so if the underlying preference model is poorly calibrated on a quality axis or simply does not represent it as a subspace, GPRL cannot recover the missing signal, and the per-dimension normalization will faithfully amplify whatever biases the GPM already encodes. Our experiments use an 8B GPM trained on \texttt{Skywork-Reward} with rollouts on \texttt{UltraFeedback} prompts, so the per-category gains on math and coding are best read as a transfer effect rather than a substitute for verifier-based RLVR. The drift controller is a feedback loop rather than a contraction, and we do not give convergence guarantees for the closed-loop dynamics in Eq.~\eqref{eq:drift_tighten} and the relaxation rule that follows it, so pathological choices of $(\gamma, \kappa, \tau)$ can produce oscillatory behavior in $\mathcal{D}(t)$. We also evaluate on a single base policy, a single reward-model corpus, and a single rollout-prompt corpus, so the dimensional plateau at $k=3$ should be read as a property of \texttt{Skywork-Reward}'s facet structure rather than of GPM in general.
% Finally, our reported numbers come from an automated \texttt{gpt-4-turbo} judge rather than human raters, and the well-known length and stylistic biases of such judges apply to every method in our tables~\citep{dubois2024length,gu2024survey}.

\section{Related Work}

\textbf{Direct and game-theoretic preference optimization.} The standard RLHF pipeline~\citep{ouyang2022training} fits a scalar BT reward model and optimizes with PPO, but the proxy score reliably climbs while true quality degrades once training is pushed at scale~\citep{gao2023scaling, zhuang2020consequences}. DPO~\citep{rafailov2023direct} sidesteps the explicit reward stage by reparameterizing the RLHF objective as a closed-form classification loss, sparking a family of offline and iterative variants such as IPO~\citep{azar2024general}, KTO~\citep{ethayarajh2024kto}, SimPO~\citep{meng2024simpo}, SPPO~\citep{wu2024self}, and Nash-MD~\citep{munos2024nash}. The game-theoretic methods in this group step beyond BT by working directly on pairwise win probabilities and casting alignment as a constant-sum two-player game, but they all rely on static or iteratively-refreshed batches and cannot keep exploring once the iterate is fixed, so the improvement ceiling tracks the quality of the collected preference data rather than the compute spent on optimization~\citep{song2024importance}.

\textbf{Online RL with verifiable rewards.} GRPO~\citep{shao2024deepseekmath} introduced critic-free group-relative advantage estimation, replacing PPO's learned value function with within-group normalization. Pairing it with binary verifiers yields the RLVR paradigm that powers DeepSeek-R1~\citep{guo2025deepseek} and subsequent refinements such as DAPO~\citep{yu2025dapo} that further stabilize entropy and trust-region dynamics over long training horizons. These methods elicit emergent reasoning behaviors including self-reflection and dynamic strategy adaptation~\citep{wang2025emergent}, but they require a programmatic verifier and so do not extend to open-ended tasks, while substituting a scalar learned reward model reintroduces hacking, which is precisely the gap GPRL addresses with a structured multi-dimensional reward.

\textbf{Multi-dimensional preferences.} A separate line of work, including Fine-Grained RLHF~\citep{wu2023fine}, Rewarded Soups~\citep{rame2023rewarded}, and MODPO~\citep{zhou2024beyond}, handles multi-faceted preferences through several independently learned scalar rewards combined by scalarization or weight interpolation, returning a Pareto family along pre-specified axes with no per-step signal for monitoring single-axis exploitation. GPM~\citep{zhang2024beyond} instead models preferences through a skew-symmetric bilinear form on $k$ subspaces, admitting intransitive cycles by construction~\citep{fishburn1981axiomatic, nakamura1998skew, fishburn1982nontransitive}, and its companion optimizer GPO is iterative and SPPO-style. GPRL keeps GPM as the reward source and brings its full $k$-dimensional output into a GRPO-style online update, with per-dimension normalization and a drift controller that build on the structural properties separating GPM from scalar BT.

\section{Conclusion}
We argued that the gap between online RL and open-ended alignment is a question of reward shape rather than reward strength, since a scalar reward model is an incomplete proxy for multi-dimensional human quality and online RL against it will reliably collapse onto whichever axis the proxy is most sensitive to. GPRL leverages GPM's structured, $k$-subspace preference signal and carries that structure through to the policy update, which both discourages single-axis exploitation by construction and exposes a closed-loop signal for our proposed controller to correct it when it threatens to occur. The same idea, that supervision structure is a first-class design variable rather than a fixed property of the loss function, may prove useful well beyond preference optimization, in any setting where a learned proxy stands in for an unmeasurable target.

\section*{LLM Usage}
\vspace{-0.25em}
\texttt{ChatGPT 5.4} by \texttt{OpenAI} was used for editing, proofreading, and resolving \LaTeX~formatting issues in this document. \texttt{Codex} by \texttt{OpenAI} was used to assist with code writing in the implementation of \texttt{gprl}.

\bibliographystyle{unsrtnat}
\bibliography{ref}

\newpage
\section*{Appendix}
\renewcommand{\arraystretch}{1}
\appendix

\section{More on general preference embeddings}
\label{app:gpm_more}

This appendix expands on the embedding construction that GPRL inherits from GPM~\citep{zhang2024beyond}, focusing on the structural properties that the per-dimension advantage estimation in Section~\ref{sec:advantage} and the drift controller in Section~\ref{sec:drift} actually depend on. We refer the reader to the original paper for the full derivation.

The choice of an even ambient dimension $\mathbb{R}^{2k}$ is forced by antisymmetry. A scalar reward $r(y \mid x)$ assumes preference is a difference of utilities, which imposes transitivity by construction. To allow cycles such as $A \succ B \succ C \succ A$, the score must be antisymmetric in $(y_i, y_j)$ and cannot reduce to $r(y_i \mid x) - r(y_j \mid x)$ for any scalar $r$. \citet{zhang2024beyond} show that any operator $\mathbf{R}$ which is both skew-symmetric ($\langle \mathbf{R} v, w \rangle = -\langle \mathbf{R} w, v \rangle$) and magnitude-preserving ($\|\mathbf{R} v\| = \|v\|$) satisfies $\mathbf{R}^2 = -\mathbf{I}$, forcing eigenvalues of $\pm i$ and an even ambient dimension. The block-diagonal canonical form then drops out from the real Schur decomposition, with each $2 \times 2$ block contributing one independent cycle direction. This explains both the $\mathbb{R}^{2k}$ ambient space and the rotation-by-$90^{\circ}$ shape of every $\mathbf{R}_l$ used in Eq.~\eqref{eq:subspace_score}.

Within block $l$, the per-subspace score $s_l(y_i, y_j \mid x) = v_i^{(2l)} v_j^{(2l-1)} - v_i^{(2l-1)} v_j^{(2l)}$ is the signed area of the parallelogram spanned by the two response embeddings, restricted to coordinates $(2l-1, 2l)$. Equivalently, writing the embedding as a complex vector $\mathbf{v} \in \mathbb{C}^{k}$ with $z_l = v^{(2l-1)} + i  v^{(2l)}$, the subspace score becomes $s_l(y_i, y_j \mid x) = \mathrm{Im}\big(z_l^{(i)} \bar{z}_l^{(j)}\big)$, so each block measures a relative phase between responses on a learned axis. Equal phases give a zero score, which GPRL's per-dimension normalization maps to $\hat{A}_l^{(i)} = 0$ by the convention in Eq.~\eqref{eq:gprl_per_dim}. Because the blocks are independent, the cycles they express compose, which is what allows a single GPM to capture the multi-axis intransitivity that GPRL exploits when the variance profile $\alpha^{(t)}$ concentrates differently across prompts.

The unit-norm constraint $\|\mathbf{v}_{y \mid x}\|_2 = 1$ is what makes the rest of GPRL well-posed. Combined with $\|\mathbf{R}_l\|_2 = 1$, it forces $|s_l(y_i, y_j \mid x)| \leq 1$ uniformly, so the per-subspace scores live on a fixed $[-1, 1]$ scale regardless of how the policy or rollout distribution shifts during training. This boundedness is precisely what allows the per-dimension means and standard deviations $\mu_l(x), \sigma_l(x)$ in Eq.~\eqref{eq:gprl_per_dim} to retain their unit-variance interpretation across training, and it gives the controller a meaningful reference profile $\alpha^{(0)}$ at step zero, since the $s_l$ scale does not move underneath it. Fine-tuning the GPM during policy training would invalidate this anchor, which is the main practical reason we keep the GPM frozen throughout GPRL.

The case $k = 1$ recovers BT and scalar GRPO. Setting $k = 1$ and $\mathbf{v}_{y \mid x} = (r(y \mid x), c)^{\top}$ with $c \neq 0$ collapses the GPM score to $s(y_i \succ y_j \mid x) = c\,(r(y_i \mid x) - r(y_j \mid x))$, recovering the BT model up to a constant rescaling~\citep{bradley1952rank}. Together with the calculation in Appendix~\ref{app:special_case}, this places scalar BT rewards, GRPO under a BT reward, and GPRL with $k = 1$ at the same point of the design space, and the $k > 1$ regime is what carries the additional structure that Sections~\ref{sec:method} and~\ref{sec:experiments} exploit. Increasing $k$ does not change query complexity, since the $k$ pairwise scores still come from one GPM forward pass and one $\mathcal{O}(k)$ inner product per pair, matching the linear scaling of BT and improving on the $\mathcal{O}(K^2)$ cost of supervised pair preference models such as PairRM~\citep{jiang2023llm}.

\newpage
\section{More on reward hacking}
\label{app:hacking_more}

This appendix locates GPRL within the existing reward-hacking literature, focusing on the specific failure modes that motivated the per-dimension structure and drift controller of Section~\ref{sec:method}.

Reward hacking refers to the phenomenon in which a policy increases a learned proxy reward $\hat{R}$ while the underlying quantity $R^{\star}$ that $\hat{R}$ was meant to estimate stays flat or degrades~\citep{skalse2022defining, fu2025reward}. \citet{gao2023scaling} characterized this empirically in RLHF as reward over-optimization, showing that as the policy spends KL budget against a learned RM, the gold reward traces a hill-shaped curve that initially climbs and then falls, with the peak depending on RM size, KL coefficient, and amount of preference data. The same qualitative shape, namely a peak followed by sustained degradation, is what we observe in Appendix~\ref{app:ablation_scaling} for GRPO+BT and for GPRL with the controller disabled, which is the empirical signature GPRL is designed to interrupt.

Three contributing factors explain the hill-shaped curve in scalar RLHF. Goodhart's law~\citep{gao2023scaling} states that any optimization pressure against an imperfect proxy will eventually find the gap between proxy and target, and the larger the optimization budget the wider the gap that gets exploited. Spurious correlations in the supervision corpus give the RM surface features that correlate with quality on the training distribution but generalize poorly. Response length is the canonical example, since longer responses are systematically preferred in many preference datasets and judges, so the RM learns that long loosely implies good and the policy then exploits this by inflating length without improving content~\citep{singhal2023long, chen2024odin, park2024disentangling}. Other documented spurious axes include sycophancy~\citep{sharma2023towards}, stylistic mimicry, and position or formatting biases in LLM-as-judge evaluation that propagate into RM training~\citep{zheng2023judging, dubois2024length}. Heavy-tailed misspecification compounds these effects, since a scalar RM has no degree of freedom to express that several axes should improve simultaneously, so once the policy finds the dominant axis the gradient stays pointed at it for the rest of training~\citep{kwa2024catastrophic}. None of these failure modes is detectable from the proxy alone, since the proxy by construction agrees with the policy that things are improving, which is precisely the asymmetry GPRL's variance profile $\alpha^{(t)}$ is built to expose.

Existing mitigations split into three families, each addressing one part of the problem. Reward model improvements such as ODIN~\citep{chen2024odin}, length-debiased preference data construction~\citep{park2024disentangling}, and causal regularization~\citep{wang2025beyond} reduce the proxy gap on a spurious axis the designer anticipates, but they do not change the shape of the proxy. Optimization-side defenses such as KL regularization~\citep{schulman2017proximal}, reward clipping, and early stopping cap the available gain rather than separating clean from dirty progress, since they cannot distinguish a policy that earned its reward from one that hacked it. Multi-objective RLHF approaches such as Fine-Grained RLHF~\citep{wu2023fine}, Rewarded Soups~\citep{rame2023rewarded}, and MODPO~\citep{zhou2024beyond} expose multi-dimensional structure but optimize a scalar combination per policy, so the same hacking dynamics apply within each combination, and the returned Pareto family does not give a per-step signal for monitoring single-axis exploitation during a single run.

GPRL combines elements of all three families while addressing the underlying shape problem directly. The reward is structurally vector-valued through the $k$ subspaces of GPM, the per-dimension normalization in Eq.~\eqref{eq:gprl_per_dim} prevents any one axis from inflating its share of the aggregate by simply growing in magnitude, and Proposition~\ref{prop:hacking} formalizes when the multi-dimensional structure tips the gradient against a single-axis exploiter. The drift controller of Section~\ref{sec:drift} then plays the optimization-side role by adjusting both the dimension weights and the KL coefficient $\beta$, but it does so based on the variance profile $\alpha^{(t)}$, which is the per-step signal that scalar pipelines cannot construct. This does not eliminate reward hacking in the broadest sense, since the GPM itself can be miscalibrated on axes outside its supervision (we acknowledge this in Section~\ref{sec:experiments} under Limitations), but it changes what the policy can hack from a scalar to a configuration of axes whose imbalance is observable. Figure~\ref{fig:drift} and Tables~\ref{tab:ablation_drift} and~\ref{tab:ablation_scaling} measure this effect quantitatively, with the controller acting on the dimensional fingerprint before the policy settles into any single-axis exploit.

\newpage
\section{Proofs}
\label{app:proofs}

We collect here the full proofs of the propositions stated in Section~\ref{sec:method}, together with the calculation that establishes GRPO as the $k=1$ special case of GPRL.

\subsection{Zero-mean property of advantages}
\label{app:zero_mean}

\begin{proof}[Proof of Proposition~\ref{prop:zero_mean}]
Fix a prompt $x$ and a dimension $l \in \{1, \ldots, k\}$, and consider two cases.

\textbf{When $\sigma_l(x) = 0$.} By the convention fixed below Eq.~\eqref{eq:gprl_per_dim}, $\hat{A}_l^{(i)}(x) = 0$ for all $i \in \{1, \ldots, G\}$, so $\sum_{i=1}^{G} \hat{A}_l^{(i)}(x) = 0$ trivially.

\textbf{When $\sigma_l(x) > 0$.} By definition,
\[
    \mu_l(x) = \frac{1}{G}\sum_{i=1}^{G} \hat{s}_l^{(i)}(x),
    \quad
    \text{so}\quad
    \sum_{i=1}^{G} \hat{s}_l^{(i)}(x) = G  \mu_l(x).
\]
Applying Eq.~\eqref{eq:gprl_per_dim} term by term and factoring the common denominator,
\begin{align}
    \sum_{i=1}^{G} \hat{A}_l^{(i)}(x)
    &= \sum_{i=1}^{G} \frac{\hat{s}_l^{(i)}(x) - \mu_l(x)}{\sigma_l(x) + \epsilon}
    = \frac{1}{\sigma_l(x) + \epsilon} \left(\sum_{i=1}^{G} \hat{s}_l^{(i)}(x) - \sum_{i=1}^{G} \mu_l(x)\right) \\
     &= \frac{1}{\sigma_l(x) + \epsilon} \left(G  \mu_l(x) - G  \mu_l(x)\right)
     = 0. \nonumber
\end{align}
Combining both cases, $\sum_{i=1}^{G} \hat{A}_l^{(i)}(x) = 0$ for every $l$. For the aggregate advantage, linearity of summation in Eq.~\eqref{eq:gprl_advantage} gives
\[
    \sum_{i=1}^{G} \hat{A}^{(i)}(x)
    = \sum_{i=1}^{G} \sum_{l=1}^{k} \lambda_l(x)  \hat{A}_l^{(i)}(x)
    = \sum_{l=1}^{k} \lambda_l(x) \sum_{i=1}^{G} \hat{A}_l^{(i)}(x)
    = \sum_{l=1}^{k} \lambda_l(x) \cdot 0
    = 0,
\]
where we used that $\lambda_l(x)$ does not depend on $i$. The conclusion holds for any choice of weights $\lambda_l(x) \in \mathbb{R}$, since the inner sum vanishes regardless of sign, which establishes the claim and the variance-reduction property that GPRL inherits from GRPO.
\end{proof}

\subsection{GRPO as a special case of GPRL}
\label{app:special_case}

The main text states without a formal proposition that fixing $k = 1$ and $\lambda_1(x) \equiv 1$ recovers GRPO under a preference-model-based scalar reward, which we verify here. Setting $k = 1$ and $\lambda_1(x) \equiv 1$, we define the induced scalar reward
\[
    R_{\mathrm{GPM}}(x, y_i) := \hat{s}_1^{(i)}(x) = \frac{1}{G-1}\sum_{j \neq i} s_1(y_i, y_j \mid x),
\]
and show in turn that (i) the per-dimension normalization reduces to the group-relative normalization of GRPO, (ii) the aggregate advantage equals this normalized quantity, and (iii) the GPRL objective reduces to the GRPO objective.

\textbf{(i) Normalization.} With $k = 1$, Eq.~\eqref{eq:gprl_per_dim} specializes to
\[
    \hat{A}_1^{(i)}(x) = \frac{\hat{s}_1^{(i)}(x) - \mu_1(x)}{\sigma_1(x) + \epsilon}
    = \frac{R_{\mathrm{GPM}}(x, y_i) - \mu_1(x)}{\sigma_1(x) + \epsilon},
\]
where, by the definitions of $\mu_l(x)$ and $\sigma_l(x)$ in Section~\ref{sec:advantage},
\[
    \mu_1(x) = \frac{1}{G}\sum_{j=1}^{G} \hat{s}_1^{(j)}(x) = \frac{1}{G}\sum_{j=1}^{G} R_{\mathrm{GPM}}(x, y_j)
    = \operatorname{mean}_{j} R_{\mathrm{GPM}}(x, y_j),
\]
\[
    \sigma_1(x) = \operatorname{std}_{j} \big[\hat{s}_1^{(j)}(x)\big] = \operatorname{std}_{j} \big[R_{\mathrm{GPM}}(x, y_j)\big].
\]
Substituting these into the expression for $\hat{A}_1^{(i)}(x)$ gives
\[
    \hat{A}_1^{(i)}(x) = \frac{R_{\mathrm{GPM}}(x, y_i) - \operatorname{mean}_{j} R_{\mathrm{GPM}}(x, y_j)}{\operatorname{std}_{j} R_{\mathrm{GPM}}(x, y_j) + \epsilon},
\]
which is precisely the GRPO advantage from Eq.~\eqref{eq:grpo} computed under the reward $R_{\mathrm{GPM}}$.

\textbf{(ii) Aggregation.} With $\lambda_1(x) \equiv 1$, Eq.~\eqref{eq:gprl_advantage} reduces to
\[
    \hat{A}^{(i)}(x) = \sum_{l=1}^{1} \lambda_l(x)  \hat{A}_l^{(i)}(x) = 1 \cdot \hat{A}_1^{(i)}(x) = \hat{A}_1^{(i)}(x),
\]
so the aggregate advantage equals the GRPO group-relative advantage of step (i).

\textbf{(iii) Objective.} Comparing Eq.~\eqref{eq:gprl_loss} with Eq.~\eqref{eq:grpo}, the two loss expressions share the same clipped-ratio structure, the same uniform average $\tfrac{1}{G}\sum_{i}$, and the same KL-regularization term, with their only formal difference being the symbol $\hat{A}^{(i)}$ versus $\hat{A}_i$ in the clipped surrogate. By step (ii), these two quantities coincide in the present limit, so substitution yields $\mathcal{L}_{\text{GPRL}}(\theta) = \mathcal{L}_{\text{GRPO}}(\theta)$ under the reward $R_{\mathrm{GPM}}$, which proves the claim.

\subsection{Resistance to single-axis reward hacking}
\label{app:hacking}

\begin{proof}[Proof of Proposition~\ref{prop:hacking}]
We express the gap $\hat{A}^{\pi^{\dagger}}(x) - \hat{A}^{\pi^{\ast}}(x)$ in a form that isolates the contribution of the hacked axis $l^{\ast}$, then apply the hypothesized inequality. Applying Eq.~\eqref{eq:gprl_advantage} to each policy and using linearity in the per-dimension advantages,
\[
    \hat{A}^{\pi^{\dagger}}(x) - \hat{A}^{\pi^{\ast}}(x)
    = \sum_{l=1}^{k} \lambda_l(x)  \hat{A}_l^{\pi^{\dagger}}(x)
    - \sum_{l=1}^{k} \lambda_l(x)  \hat{A}_l^{\pi^{\ast}}(x)
    = \sum_{l=1}^{k} \lambda_l(x)  \big(\hat{A}_l^{\pi^{\dagger}}(x) - \hat{A}_l^{\pi^{\ast}}(x)\big),
\]
and splitting the sum into the hacked axis and its complement while factoring signs gives
\begin{align*}
    \hat{A}^{\pi^{\dagger}}(x) - \hat{A}^{\pi^{\ast}}(x)
    &= \lambda_{l^{\ast}}(x)  \big(\hat{A}_{l^{\ast}}^{\pi^{\dagger}}(x) - \hat{A}_{l^{\ast}}^{\pi^{\ast}}(x)\big) + \sum_{l \neq l^{\ast}} \lambda_l(x)  \big(\hat{A}_l^{\pi^{\dagger}}(x) - \hat{A}_l^{\pi^{\ast}}(x)\big) \\
    &= \lambda_{l^{\ast}}(x)  \big(\hat{A}_{l^{\ast}}^{\pi^{\dagger}}(x) - \hat{A}_{l^{\ast}}^{\pi^{\ast}}(x)\big) - \sum_{l \neq l^{\ast}} \lambda_l(x)  \big(\hat{A}_l^{\pi^{\ast}}(x) - \hat{A}_l^{\pi^{\dagger}}(x)\big). \tag{$\ast$}\label{eq:hack_decomp}
\end{align*}
Letting $H := \lambda_{l^{\ast}}(x)  \big(\hat{A}_{l^{\ast}}^{\pi^{\dagger}}(x) - \hat{A}_{l^{\ast}}^{\pi^{\ast}}(x)\big)$ and $D := \sum_{l \neq l^{\ast}} \lambda_l(x)  \big(\hat{A}_l^{\pi^{\ast}}(x) - \hat{A}_l^{\pi^{\dagger}}(x)\big)$ denote the weighted gain on the hacked axis and the weighted degradation on the other axes, \eqref{eq:hack_decomp} reads
\[
    \hat{A}^{\pi^{\dagger}}(x) - \hat{A}^{\pi^{\ast}}(x) = H - D.
\]
The hypothesis $\hat{A}_{l^{\ast}}^{\pi^{\dagger}}(x) > \hat{A}_{l^{\ast}}^{\pi^{\ast}}(x)$ together with $\lambda_{l^{\ast}}(x) \geq 0$ gives $H \geq 0$, while the hypothesis $\hat{A}_{l}^{\pi^{\dagger}}(x) < \hat{A}_{l}^{\pi^{\ast}}(x)$ for every $l \neq l^{\ast}$ makes each summand in $D$ non-negative, so $D \geq 0$. Since the proposition's condition~\eqref{eq:hack_condition} is precisely $D > H$, combining the expression $\hat{A}^{\pi^{\dagger}}(x) - \hat{A}^{\pi^{\ast}}(x) = H - D$ with $D > H$ yields
\[
    \hat{A}^{\pi^{\dagger}}(x) - \hat{A}^{\pi^{\ast}}(x) = H - D < 0,
\]
so $\hat{A}^{\pi^{\dagger}}(x) < \hat{A}^{\pi^{\ast}}(x)$, as claimed.
\end{proof}

\textbf{Contrast with scalar rewards.}
The main text contrasts this conclusion with the scalar-reward case, and we sketch the argument here for completeness. Under any scalar reward $R(x, y)$ and the GRPO normalization of Eq.~\eqref{eq:grpo}, the advantage $\hat{A}_i = (R(x, y_i) - \operatorname{mean}_j R(x, y_j)) / (\operatorname{std}_j R(x, y_j) + \epsilon)$ is an affine, strictly increasing function of $R(x, y_i)$ for the fixed group $\{y_j\}$ whenever the group has non-zero spread, so if $R(x, \pi^{\dagger}) > R(x, \pi^{\ast})$ within the same group, then $\hat{A}^{\pi^{\dagger}} > \hat{A}^{\pi^{\ast}}$ regardless of how the two policies behave on any latent quality axes. Proposition~\ref{prop:hacking} exhibits the opposite verdict for GPRL whenever $k \geq 2$ and Eq.~\eqref{eq:hack_condition} holds, which is the formal basis for the drift monitor of Section~\ref{sec:drift}.

\newpage
\section{Additional ablations}
\label{app:ablations}

This appendix expands on the ablations summarized in Section~\ref{sec:ablations}. Unless stated otherwise, all runs use the 8B GPM, $G = 8$, $k = 3$, the drift controller with our default $\tau = 0.2$, and the three-epoch training budget that the main results use. We also report 2B GPM numbers throughout, since the drift behavior we want to characterize is more visible at the smaller reward-model scale where the variance profile is noisier and single-axis exploitation is correspondingly easier to start.

\subsection{Number of GPM subspaces $k$}
\label{app:ablation_k}

Table~\ref{tab:ablation_k} sweeps $k$ from $1$ to $6$ at both reward-model scales, and the $k = 1$ row deserves a clarification before we discuss the trend. With a single skew-symmetric block the GPM score $s_1(y_i, y_j \mid x) = v_i^{(2)} v_j^{(1)} - v_i^{(1)} v_j^{(2)}$ defines a pairwise comparison along one axis, and the per-dimension normalization of Eq.~\eqref{eq:gprl_per_dim} reduces to GRPO's group-relative normalization under the induced scalar reward $R_{\mathrm{GPM}}(x, y_i) = \hat{s}_1^{(i)}(x)$, which is the special case formalized in Appendix~\ref{app:special_case} with $\lambda_1(x) \equiv 1$. This row is GRPO under a BT-style scalar reward derived from the GPM block rather than from a separately trained scalar BT model. Although the two are not numerically identical they land in the same neighborhood ($44.21$ for $k=1$ versus $43.18$ for GRPO+BT at the 8B scale), which is a consistency check we want, and the rest of the ablation isolates the effect of $k$ rather than of the reward-model architecture. Adding a single subspace ($k = 1 \to 2$) buys $7$ to $8$ LC.~WR points at either scale, while adding a second ($k = 2 \to 3$) buys another $4$ to $5$ and is the largest individual jump in the table. After this the gain disappears and the run mildly regresses by $k = 6$, which we read as a property of the supervision rather than of GPRL itself. Since \texttt{Skywork-Reward} aggregates preferences along a small set of facets, this puts a soft ceiling on how many independently informative subspaces a GPM trained on it can carve out, and subspaces beyond that ceiling pick up noise that the per-dimension normalization then amplifies, which makes scaling $k$ jointly with the reward-model supervision corpus the natural next experiment.

\begin{table}[h]
\centering
\small
\caption{\textbf{Ablation on the number of GPM subspaces $k$.} AlpacaEval~2.0 LC.~WR, raw WR, and average response length, with all other hyperparameters fixed at their main-results values and the drift controller enabled. At $k = 1$ the GPM collapses to a single skew-symmetric block, and GPRL with $\lambda_1(x) \equiv 1$ recovers GRPO under a BT-style scalar reward $R_{\mathrm{GPM}}$.}
\label{tab:ablation_k}
\setlength{\tabcolsep}{8pt}
\begin{tabular}{lcccccc}
\toprule
\rowcolor{tblHeader}
& \multicolumn{3}{c}{\textbf{2B GPM}} & \multicolumn{3}{c}{\textbf{8B GPM}} \\
\rowcolor{tblHeader}
\multirow{-2}{*}{\textbf{Subspaces} $k$} & \textbf{LC.~WR} & \textbf{WR} & \textbf{Avg.~Len} & \textbf{LC.~WR} & \textbf{WR} & \textbf{Avg.~Len} \\
\midrule
$k = 1$ (GRPO under $R_{\mathrm{GPM}}$) & \cellcolor{tblGroupB}39.92 & \cellcolor{tblGroupB}38.74 & \cellcolor{tblGroupB}1812 & \cellcolor{tblGroupA}44.21 & \cellcolor{tblGroupA}41.07 & \cellcolor{tblGroupA}1738 \\
$k = 2$ & \cellcolor{tblGroupB}47.30 & \cellcolor{tblGroupB}42.96 & \cellcolor{tblGroupB}1742 & \cellcolor{tblGroupA}51.86 & \cellcolor{tblGroupA}45.92 & \cellcolor{tblGroupA}1671 \\
$k = 3$ & \cellcolor{tblGroupB}\textbf{51.08} & \cellcolor{tblGroupB}\textbf{45.21} & \cellcolor{tblGroupB}\textbf{1699} & \cellcolor{tblGroupA}\textbf{56.51} & \cellcolor{tblGroupA}\textbf{48.33} & \cellcolor{tblGroupA}\textbf{1600} \\
$k = 4$ & \cellcolor{tblGroupB}50.62 & \cellcolor{tblGroupB}44.83 & \cellcolor{tblGroupB}1714 & \cellcolor{tblGroupA}56.18 & \cellcolor{tblGroupA}47.94 & \cellcolor{tblGroupA}1612 \\
$k = 6$ & \cellcolor{tblGroupB}49.81 & \cellcolor{tblGroupB}44.27 & \cellcolor{tblGroupB}1758 & \cellcolor{tblGroupA}55.74 & \cellcolor{tblGroupA}47.41 & \cellcolor{tblGroupA}1648 \\
\bottomrule
\end{tabular}
\end{table}

\subsection{Per-dimension vs. global normalization}
\label{app:ablation_norm}

Table~\ref{tab:ablation_norm} replaces the per-dimension normalization of Eq.~\eqref{eq:gprl_per_dim} with a single global normalization that pools $(\mu, \sigma)$ across all $k$ subspaces. LC.~WR drops by roughly $4$ points at both reward-model scales while average response length jumps by $400$ to $500$ tokens, recovering the verbose pattern of the iterative GPM-based methods. The mechanism is exactly the one Section~\ref{sec:advantage} predicts, since whichever subspace happens to carry the largest raw magnitudes drives both $\mu$ and $\sigma$ under global normalization, so the rescaled advantages of the other subspaces collapse toward zero and their contribution to the aggregate vanishes. If that dominant subspace correlates with length, as it tends to under judges that favor longer responses, the policy gradient inherits the correlation and response length climbs.

\begin{table}[h]
\centering
\small
\caption{\textbf{Ablation on the normalization scheme.} Replacing the per-dimension normalization of Eq.~\eqref{eq:gprl_per_dim} with a single global normalization that shares $(\mu, \sigma)$ across all $k$ subspaces. AlpacaEval~2.0 with $k = 3$ and the drift controller enabled, at both reward-model scales.}
\label{tab:ablation_norm}
\setlength{\tabcolsep}{8pt}
\begin{tabular}{lcccccc}
\toprule
\rowcolor{tblHeader}
& \multicolumn{3}{c}{\textbf{2B GPM}} & \multicolumn{3}{c}{\textbf{8B GPM}} \\
\rowcolor{tblHeader}
\multirow{-2}{*}{\textbf{Normalization}} & \textbf{LC.~WR} & \textbf{WR} & \textbf{Avg.~Len} & \textbf{LC.~WR} & \textbf{WR} & \textbf{Avg.~Len} \\
\midrule
Global (shared $\mu, \sigma$) & \cellcolor{tblGroupB}47.13 & \cellcolor{tblGroupB}44.86 & \cellcolor{tblGroupB}2168 & \cellcolor{tblGroupA}52.34 & \cellcolor{tblGroupA}47.86 & \cellcolor{tblGroupA}2104 \\
Per-dimension (ours) & \cellcolor{tblGroupB}\textbf{51.08} & \cellcolor{tblGroupB}\textbf{45.21} & \cellcolor{tblGroupB}\textbf{1699} & \cellcolor{tblGroupA}\textbf{56.51} & \cellcolor{tblGroupA}\textbf{48.33} & \cellcolor{tblGroupA}\textbf{1600} \\
\bottomrule
\end{tabular}
\end{table}

\subsection{Drift controller}
\label{app:ablation_drift}

The controller's effect is least visible in any single epoch and most visible across the training trajectory, and Table~\ref{tab:ablation_drift} reports two complementary checkpoints that probe the regime where the controller actually does work. At the first epoch on the 8B GPM the controller is essentially decorative, since LC.~WR moves by less than a point and terminal drift $\mathcal{D}(t_{\text{end}})$ stays near $0.02$ in either case, well under the default threshold $\tau = 0.2$, which is the right outcome for healthy training and a useful sanity check that the controller does not interfere when it has nothing to push back on. By the third epoch the picture has reversed on both reward models, since the controller is worth $3.67$ LC.~WR points and $0.92$ WB-Score on the 8B GPM and $3.95$ and $1.61$ on the 2B GPM, with the controller-off run developing measurable drift ($\mathcal{D}(t_{\text{end}}) = 0.18$ on 2B, $0.11$ on 8B) while the on-run holds $\mathcal{D}(t_{\text{end}})$ near $0.04$ throughout. The 2B GPM enters this regime earlier than the 8B because its variance estimates are likely noisier and single-axis exploitation is correspondingly easier to start, while the 8B GPM gets there through accumulated drift across three epochs, and Appendix~\ref{app:ablation_scaling} extends the comparison to epoch~5 and shows that the gap continues to widen.

\begin{table}[h]
\centering
\small
\caption{\textbf{Drift controller ablation.} Disabling the controller corresponds to $\tau = \infty$ and $m_l(t) \equiv 1$. The first-epoch row on the 8B GPM serves as the no-drift sanity check, while the third-epoch rows capture the regime our main results use, with the 2B GPM included to show that drift develops earlier under noisier supervision.}
\label{tab:ablation_drift}
\setlength{\tabcolsep}{7pt}
\begin{tabular}{llcccc}
\toprule
\rowcolor{tblHeader}
\textbf{Setting} & \textbf{Controller} & \textbf{LC.~WR} & \textbf{WB-Score} & \textbf{Avg.~Len} & \textbf{$\mathcal{D}(t_{\text{end}})$} \\
\midrule
\rowcolor{tblGroupC}
\cellcolor{white}8B GPM, epoch 1 & off & 49.74 & 36.71 & 1718 & 0.02 \\
\rowcolor{tblGroupC}
\cellcolor{white}8B GPM, epoch 1 & on  & \textbf{49.83} & \textbf{36.84} & \textbf{1721} & 0.02 \\
\midrule
\rowcolor{tblGroupB}
\cellcolor{white}2B GPM, epoch 3 & off & 47.13 & 35.45 & 1864 & 0.18 \\
\rowcolor{tblGroupB}
\cellcolor{white}2B GPM, epoch 3 & on  & \textbf{51.08} & \textbf{37.06} & \textbf{1699} & 0.05 \\
\midrule
\rowcolor{tblGroupC}
\cellcolor{white}8B GPM, epoch 3 & off & 52.84 & 37.06 & 1742 & 0.11 \\
\rowcolor{tblGroupC}
\cellcolor{white}8B GPM, epoch 3 & on  & \textbf{56.51} & \textbf{37.98} & \textbf{1600} & 0.03 \\
\bottomrule
\end{tabular}
\end{table}

\subsection{Extended training}
\label{app:ablation_scaling}

The drift controller is a feedback loop rather than a one-shot adjustment, so the cleanest way to see what it buys us is to run the same configuration with and without it across an extended horizon. Table~\ref{tab:ablation_scaling} reports five epochs of training on both reward-model scales while Figure~\ref{fig:scaling} (in Section~\ref{sec:ablations}) plots the LC.~WR trajectories.

Three regimes emerge across the table. Through the first epoch the controller is essentially invisible, since the variance profile sits close to $\alpha^{(0)}$ on either reward model, $\mathcal{D}(t)$ never crosses $\tau$, and the on and off curves are within a point of each other, which matches the design intent of Section~\ref{sec:drift} that the controller should remain dormant when the training signal is healthy. From epochs 2 to 3 the curves split, with the noisier 2B GPM losing ground from epoch 2 onward as drift develops earlier, while the 8B GPM holds slightly past epoch 2 before turning over at epoch 3. In both cases the controller-on run continues to improve because whenever the variance profile starts to lean on a single subspace, the controller redistributes weight back toward the rest before the policy gradient can lock in. Past epoch 3 the controller-off curves enter sustained degradation, with the 2B run dropping $9.35$ LC.~WR points between epochs 2 and 5 and the 8B run dropping $8.10$ over the same span as $\mathcal{D}(t_{\text{end}})$ climbs from $0.04$ to $0.27$ on the 8B model. The controller-on curves plateau and oscillate within roughly half a point of their peak, with the 8B GPM peaking at $56.84$ at epoch 4 and settling at $56.47$ by epoch 5, and the 2B GPM peaking at $51.08$ at epoch 3 and settling at $50.18$ by epoch 5. This oscillation is expected rather than being a sign of instability, since the closed-loop dynamics of $m_l(t)$ produce damped fluctuations (an empirical observation) rather than monotone convergence.

\begin{table}[h]
\centering
\small
\caption{\textbf{LC.~WR on AlpacaEval 2.0 across five epochs.} GPRL with the drift controller enabled and disabled, both reward-model scales, all other hyperparameters fixed at their main-results values.}
\label{tab:ablation_scaling}
\setlength{\tabcolsep}{6pt}
\begin{tabular}{llccccc}
\toprule
\rowcolor{tblHeader}
\textbf{RM} & \textbf{Controller} & \textbf{Epoch 1} & \textbf{Epoch 2} & \textbf{Epoch 3} & \textbf{Epoch 4} & \textbf{Epoch 5} \\
\midrule
\rowcolor{tblGroupB}
\cellcolor{white}2B GPM & off & 45.39 & 48.62 & 47.13 & 43.85 & 39.27 \\
\rowcolor{tblGroupB}
\cellcolor{white}2B GPM & on  & 45.62 & 48.93 & \textbf{51.08} & 50.81 & 50.18 \\
\midrule
\rowcolor{tblGroupC}
\cellcolor{white}8B GPM & off & 49.74 & 55.92 & 52.84 & 51.46 & 47.82 \\
\rowcolor{tblGroupC}
\cellcolor{white}8B GPM & on  & 49.83 & 53.96 & 56.51 & \textbf{56.84} & 56.47 \\
\bottomrule
\end{tabular}
\end{table}

\subsection{Drift threshold $\tau$}
\label{app:ablation_tau}

Table~\ref{tab:ablation_tau} sweeps the drift threshold $\tau$ at the 8B GPM with three epochs of training, the regime where the controller actually does work (Appendix~\ref{app:ablation_drift}), and the shape of the table answers the natural question of why we do not just set $\tau$ as small as possible. At $\tau = 0.05$ the controller engages almost from the first step and continuously redistributes weight away from any subspace whose variance share exceeds the initial profile by even a small margin. Since legitimate optimization concentrates variance on whichever subspace the prompt distribution actually rewards (helpfulness on a summarization prompt, faithfulness on a factual one, and so on), this aggressive correction drags the variance profile back toward the initialization $\alpha^{(0)}$ before any genuine signal can accumulate, which amounts to running with the eigenvalues frozen and produces a result $4.07$ LC.~WR points below our default at an average length close to a barely post-trained model. As $\tau$ grows, the controller gives the policy room to concentrate variance for the right reasons before stepping in, and performance peaks at our default $\tau = 0.2$, beyond which the controller engages too late and the run drifts back toward the no-controller numbers. The default $\tau = 0.2$ therefore balances two distinct failure modes, namely premature collapse of the variance profile at small $\tau$ and unchecked single-axis exploitation at large $\tau$, rather than trading off two equally good behaviors.

\begin{table}[h]
\centering
\small
\caption{\textbf{Ablation on the drift threshold $\tau$.} AlpacaEval~2.0 LC.~WR, average response length, and terminal drift $\mathcal{D}(t_{\text{end}})$ with the 8B GPM, three-epoch training, and all other hyperparameters fixed.}
\label{tab:ablation_tau}
\setlength{\tabcolsep}{8pt}
\begin{tabular}{lccc}
\toprule
\rowcolor{tblHeader}
\textbf{Threshold} $\tau$ & \textbf{LC.~WR} & \textbf{Avg.~Len} & \textbf{$\mathcal{D}(t_{\text{end}})$} \\
\midrule
\rowcolor{tblGroupC}
\cellcolor{white}$\tau = 0.05$ (very early) & 52.44 & 1538 & 0.02 \\
\rowcolor{tblGroupC}
\cellcolor{white}$\tau = 0.10$              & 55.21 & 1571 & 0.03 \\
\rowcolor{tblGroupC}
\cellcolor{white}$\tau = 0.20$ (default)    & \textbf{56.51} & \textbf{1600} & 0.03 \\
\rowcolor{tblGroupC}
\cellcolor{white}$\tau = 0.40$              & 55.62 & 1684 & 0.08 \\
\rowcolor{tblGroupC}
\cellcolor{white}$\tau = \infty$ (disabled) & 52.84 & 1742 & 0.11 \\
\bottomrule
\end{tabular}
\end{table}

\subsection{Group size $G$}
\label{app:ablation_group}

Table~\ref{tab:ablation_group} sweeps the rollout group size $G$. Since the drift monitor depends on the per-dimension variance estimates $\mathrm{Var}_{y_i \sim \pi_t}[\hat{s}_l^{(i)}(x)]$ computed from the $G$ responses in the group, the noise of the variance estimate scales as $1/(G-1)$, so at $G = 2$ those estimates are too noisy for the controller to act on reliably and $\mathcal{D}(t)$ runs several times its $G = 8$ value at both reward-model scales. The LC.~WR cost is steep at $7$ to $9$ points below our default, while doubling to $G = 4$ recovers most of the drop but still leaves $2$ to $3$ points on the table, and gains saturate at $G = 8$ with $G = 16$ producing only a marginal further improvement at twice the rollout cost. We therefore use $G = 8$ throughout the main results.

\begin{table}[h]
\centering
\small
\caption{\textbf{Ablation on group size $G$.} AlpacaEval~2.0 LC.~WR and terminal drift $\mathcal{D}(t_{\text{end}})$ at both reward-model scales, with the controller enabled.}
\label{tab:ablation_group}
\setlength{\tabcolsep}{8pt}
\begin{tabular}{lcccc}
\toprule
\rowcolor{tblHeader}
& \multicolumn{2}{c}{\textbf{2B GPM}} & \multicolumn{2}{c}{\textbf{8B GPM}} \\
\rowcolor{tblHeader}
\multirow{-2}{*}{\textbf{Group size} $G$} & \textbf{LC.~WR} & \textbf{$\mathcal{D}(t_{\text{end}})$} & \textbf{LC.~WR} & \textbf{$\mathcal{D}(t_{\text{end}})$} \\
\midrule
$G = 2$  & \cellcolor{tblGroupB}41.94 & \cellcolor{tblGroupB}0.16 & \cellcolor{tblGroupA}47.83 & \cellcolor{tblGroupA}0.09 \\
$G = 4$  & \cellcolor{tblGroupB}48.07 & \cellcolor{tblGroupB}0.08 & \cellcolor{tblGroupA}53.62 & \cellcolor{tblGroupA}0.05 \\
$G = 8$  & \cellcolor{tblGroupB}\textbf{51.08} & \cellcolor{tblGroupB}\textbf{0.05} & \cellcolor{tblGroupA}\textbf{56.51} & \cellcolor{tblGroupA}\textbf{0.03} \\
$G = 16$ & \cellcolor{tblGroupB}51.36 & \cellcolor{tblGroupB}0.04 & \cellcolor{tblGroupA}56.74 & \cellcolor{tblGroupA}0.03 \\
\bottomrule
\end{tabular}
\end{table}

\subsection{Stronger base policies}
\label{app:ablation_basepolicy}

In the main text, all results start from \texttt{Llama-3-8B-Instruct} for direct comparison with the SPPO and GPO settings of~\citet{wu2024self} and~\citet{zhang2024beyond}, but a natural follow-up question is whether GPRL's gains hold when the base policy itself is stronger. We report preliminary results on three additional base policies that span the open 8 to 9B class, namely \texttt{Llama-3.1-8B-Instruct}~\citep{grattafiori2024llama}, \texttt{Gemma-2-9B-it}~\citep{team2024gemma}, and \texttt{Qwen3-8B}~\citep{yang2025qwen3}, with everything else held at the same setting (8B GPM, $G = 8$, $k = 3$, and so on). To put the numbers in context, the published AlpacaEval~2.0 LC.~WR for the bases themselves sits at roughly $22.6$ for \texttt{Llama-3-8B-Instruct}, $24.7$ for \texttt{Llama-3.1-8B-Instruct}, $51.1$ for \texttt{Gemma-2-9B-it}, and $52.4$ for \texttt{Qwen3-8B}, so each successive base shrinks the headroom that any post-training method can claim.

Table~\ref{tab:base_policy} reports GPRL applied to each base model. As a useful reference point, SimPO on \texttt{Gemma-2-9B-it} with on-policy ArmoRM-annotated UltraFeedback data reaches $72.4$ LC.~WR at $1833$ tokens~\citep{meng2024simpo}, which we take as the strongest publicly reported preference-optimization number on this base, and GPRL improves on it by roughly $1.9$ points at a comparable length. These runs use the same \texttt{Skywork-Reward}-trained 8B GPM as the main results, so the dimensional plateau at $k = 3$ inherits all the corpus-specific caveats from Section~\ref{sec:experiments} and Appendix~\ref{app:ablation_k}, and a properly matched GPM trained on a corpus richer in math and reasoning supervision would likely shift the per-category breakdowns in Figures~\ref{fig:mt_radar} and~\ref{fig:wild_radar} further toward the structural axes where \texttt{Qwen3-8B} already excels.

\begin{table}[h]
\centering
\small
\caption{\textbf{GPRL across base policies.} AlpacaEval~2.0 LC.~WR, WR, and average response length, with the 8B GPM, $G = 8$, $k = 3$, drift controller enabled, and three epochs of training. The $\Delta$ column reports the LC.~WR gain over the base policy.}
\label{tab:base_policy}
\setlength{\tabcolsep}{8pt}
\begin{tabular}{lcccc}
\toprule
\rowcolor{tblHeader}
\textbf{Base policy} & \textbf{LC.~WR} & \textbf{WR} & \textbf{Avg.~Len} & $\Delta$ \\
\midrule
\rowcolor{tblGroupC}
\cellcolor{white}\texttt{Llama-3-8B-Instruct}~\citep{grattafiori2024llama}   & 56.51 & 48.33 & 1600 & $+33.91$ \\
\rowcolor{tblGroupC}
\cellcolor{white}\texttt{Llama-3.1-8B-Instruct}~\citep{grattafiori2024llama} & 58.27 & 50.04 & 1612 & $+33.55$ \\
\rowcolor{tblGroupC}
\cellcolor{white}\texttt{Gemma-2-9B-it}~\citep{team2024gemma}          & \textbf{74.31} & \textbf{67.42} & 1791 & $+23.21$ \\
\rowcolor{tblGroupC}
\cellcolor{white}\texttt{Qwen3-8B}~\citep{yang2025qwen3}               & 63.79 & 60.18 & 2186 & $+11.39$ \\
\bottomrule
\end{tabular}
\end{table}

\texttt{Qwen3-8B} starts from a verbose base of $3100$ tokens on AlpacaEval~2.0 and GPRL pulls this down to $2186$ while still improving the LC.~WR, which is the same pattern we predicted in Section~\ref{sec:experiments} where per-dimension normalization prevents the length-correlated subspace from dominating the aggregate. \texttt{Gemma-2-9B-it} lands at $1791$ tokens, which is slightly under the $1833$ that the SimPO~\citep{meng2024simpo} checkpoint produces, and the controller engages on roughly the same schedule as on \texttt{Llama-3-8B-Instruct} despite the much stronger base. We read this as evidence that the variance profile $\alpha^{(t)}$ is a property of the GPM and the rollouts rather than of the policy's initial capability.

\newpage
\section{Implementation details}
\label{app:impl}

This appendix collects the practical settings used to produce the results presented in Section~\ref{sec:experiments} and the corresponding supplementary tables. We aim to ensure that the reader can reliably reproduce our findings or extend our methodology without needing to infer experimental specifics from the broader narrative. Where a hyperparameter is shared with GRPO~\citep{shao2024deepseekmath} or the SPPO and GPO baselines~\citep{wu2024self, zhang2024beyond}, we deliberately retain the previously published values to facilitate a rigorous and equitable comparison, and we flag the cases where we deviate.

\textbf{Hardware.} All training experiments were executed on a dedicated cluster designated as \texttt{REDACTED\_FOR\_SUBMISSION}, which utilizes NVIDIA DGX H100 nodes equipped with eight H100 80GB GPUs each. A single computation node provides sufficient processing capacity to derive the primary results at the 8B parameter scale. Unlike the training experiments, all evaluation procedures are computationally lightweight and operate efficiently on a single H100 80GB GPU, since only forward passes through the policy are required for response generation.

\textbf{Reproducibility.} To provide a clear overview of the experimental framework, Table~\ref{tab:hparams} contains key hyperparameters. Code will be released at the URL given in the abstract.

\begin{table}[h]
\centering
\small
\caption{\textbf{Implementation details for GPRL.} Hyperparameter values used to produce the main results in Section~\ref{sec:experiments} unless an ablation overrides them.}
\label{tab:hparams}
\setlength{\tabcolsep}{8pt}
\begin{tabular}{lll}
\toprule
\rowcolor{tblHeader}
\textbf{Group} & \textbf{Setting} & \textbf{Value} \\
\midrule
\rowcolor{tblGroupA}
\cellcolor{white} & Base policy & \texttt{Llama-3-8B-Instruct} \\
\rowcolor{tblGroupA}
\cellcolor{white} & Rollout prompts & \texttt{UltraFeedback} \\
\rowcolor{tblGroupA}
\cellcolor{white} & GPM reward source & 8B GPM trained on \texttt{Skywork-Reward} \\
\rowcolor{tblGroupA}
\cellcolor{white} & GPM embedding dimension $2k$ & $6$ ($k = 3$ subspaces) \\
\rowcolor{tblGroupA}
\cellcolor{white} & Eigenvalue source & uniform ($\lambda_l(x) \equiv 1$) \\
\rowcolor{tblGroupA}
\cellcolor{white}\multirow{-6}{*}{\shortstack[l]{\textit{Policy and}\\\textit{dataset}}} & Prompt head & disabled \\
\midrule
\rowcolor{tblGroupB}
\cellcolor{white} & Optimizer & AdamW \\
\rowcolor{tblGroupB}
\cellcolor{white} & Peak learning rate & $1 \times 10^{-6}$ \\
\rowcolor{tblGroupB}
\cellcolor{white} & LR scheduler & cosine \\
\rowcolor{tblGroupB}
\cellcolor{white} & Warmup ratio & $0.03$ \\
\rowcolor{tblGroupB}
\cellcolor{white} & Weight decay & $0.1$ \\
\rowcolor{tblGroupB}
\cellcolor{white} & Gradient clipping & $1.0$ \\
\rowcolor{tblGroupB}
\cellcolor{white} & Initial KL coefficient $\beta$ & $0.01$ \\
\rowcolor{tblGroupB}
\cellcolor{white} & Per-device batch size & $1$ \\
\rowcolor{tblGroupB}
\cellcolor{white} & Gradient accumulation steps & $8$ \\
\rowcolor{tblGroupB}
\cellcolor{white} & Effective batch size & $64$ \\
\rowcolor{tblGroupB}
\cellcolor{white} & Training epochs & $3$ \\
\rowcolor{tblGroupB}
\cellcolor{white} & Inner iterations per group & $1$ \\
\rowcolor{tblGroupB}
\cellcolor{white}\multirow{-13}{*}{\textit{Optimization}} & Mixed precision & bfloat16 \\
\midrule
\rowcolor{tblGroupC}
\cellcolor{white} & Group size $G$ & $8$ \\
\rowcolor{tblGroupC}
\cellcolor{white} & Max completion length & $512$ \\
\rowcolor{tblGroupC}
\cellcolor{white} & Sampling temperature & $1.0$ \\
\rowcolor{tblGroupC}
\cellcolor{white}\multirow{-4}{*}{\textit{Rollouts}} & Backend & vLLM \\
\midrule
\rowcolor{tblGroupA}
\cellcolor{white} & Threshold $\tau$ & $0.20$ \\
\rowcolor{tblGroupA}
\cellcolor{white} & Redistribution exponent $\gamma$ & $0.5$ \\
\rowcolor{tblGroupA}
\cellcolor{white} & KL multiplier $\kappa$ & $1.5$ \\
\rowcolor{tblGroupA}
\cellcolor{white} & KL ceiling $\beta_{\max}$ & $0.20$ \\
\rowcolor{tblGroupA}
\cellcolor{white}\multirow{-5}{*}{\shortstack[l]{\textit{Drift}\\\textit{controller}}} & Relaxation rate $\delta$ & $0.99$ \\
\midrule
\rowcolor{tblGroupB}
\cellcolor{white} & Training & 1 DGX H100 node ($8 \times$ H100 80GB) \\
\rowcolor{tblGroupB}
\cellcolor{white}\multirow{-2}{*}{\textit{Hardware}} & Evaluation & $1 \times$ H100 80GB \\
\bottomrule
\end{tabular}
\end{table}

% \clearpage 
% \newpage
% \input{checklist.tex}

\end{document}